%% file: main.tex
\documentclass{article}

\input{definitions}

\usepackage{hyperref}

\usepackage[preprint]{icml2026}

\usepackage{mathtools}

\usepackage[capitalize,noabbrev]{cleveref}

\theoremstyle{plain}
\newtheorem{theorem}{Theorem}[section]
\newtheorem{proposition}[theorem]{Proposition}
\newtheorem{lemma}[theorem]{Lemma}

\theoremstyle{definition}

\newtheorem{assumption}[theorem]{Assumption}
\theoremstyle{remark}

\usepackage[textsize=tiny]{todonotes}

\icmltitlerunning{Unbiased Approximate Vector-Jacobian Products for Efficient Backpropagation}

\begin{document}

\twocolumn[
  \icmltitle{Unbiased Approximate Vector-Jacobian Products for Efficient Backpropagation}



  \icmlsetsymbol{equal}{*}

  \begin{icmlauthorlist}
    \icmlauthor{Killian Bakong}{inria}
    \icmlauthor{Laurent Massoulié}{inria}
    \icmlauthor{Edouard Oyallon}{edouard}
    \icmlauthor{Kevin Scaman}{inria}
  \end{icmlauthorlist}

  \icmlaffiliation{inria}{Inria Paris \& DI ENS, ENS, PSL University}
  \icmlaffiliation{edouard}{Sorbonne University, CNRS, ISIR, Paris}

  \icmlcorrespondingauthor{Killian Bakong}{firstname.lastname@inria.fr}

  \icmlkeywords{Machine Learning, ICML}

  \vskip 0.3in
]



\printAffiliationsAndNotice{}  

\begin{abstract}
  In this work we introduce methods to reduce the computational and memory costs of training deep neural networks. Our approach consists in replacing exact vector-jacobian products by randomized, unbiased approximations thereof during backpropagation. We provide a theoretical analysis of the trade-off between the number of epochs needed to achieve a target precision and the cost reduction for each epoch. We then identify specific unbiased estimates of vector-jacobian products for which we establish desirable optimality properties of minimal variance under sparsity constraints. Finally we provide in-depth experiments on multi-layer perceptrons, BagNets and Visual Transfomers architectures. These validate our theoretical results, and confirm the potential of our proposed unbiased randomized backpropagation approach for reducing the cost of deep learning.
\end{abstract}

\section{Introduction} \label{sec:intro}

Backpropagation is the workhorse of neural network training: it applies the chain rule over a computational graph to compute gradients efficiently \citep{rumelhart1986learning,lecun2002efficient,baydin2018automatic}.
Formally, it is reverse-mode automatic differentiation (AD) applied to a directed acyclic graph (DAG) of primitive operations \citep{griewank2008evaluating}.
For scalar losses, reverse-mode AD composes a sequence of vector-Jacobian products (VJPs) and is operation- and memory-optimal  \citep{linnainmaa1976taylor,baur1983complexity}.
However, despite this optimality, backpropagation remains a major speed bottleneck in modern training, which motivates numerous engineering techniques (e.g., activation checkpointing \citep{herrmann2019optimal,griewank2000algorithm,rojas2020study} and 4D parallelism \citep{li2021sequence}) to scale up those methods to more workers or larger models.

This paper studies \emph{randomized approximations of VJPs} inside reverse-mode AD.
The motivation of our method is threefold and aligns with the constraints of large-scale training, especially under pipeline parallelism \citep{huang2019gpipe} and its hybrid variants \citep{narayanan2019pipedream}.
\textbf{(i)} Communication: in pipeline parallelism, inter-layer activations often dominate cross-device traffic \citep{narayanan2019pipedream,huang2019gpipe}.
Compressing these signals while preserving gradient unbiasedness can substantially reduce bandwidth and latency.
\textbf{(ii)} Computation and memory: some randomized linear mapping (like random features sampling~\citep{rahimi2007random}) can accelerate intermediate computations (e.g., approximate VJPs \citep{coleman1983estimation}) and lower memory footprint, enabling better pipeline balance and reducing stalls by trading controlled variance for compute.
\textbf{(iii)} Deep learning specific: our framework subsumes standard gradient-compression techniques based on masking \citep{sun2017meprop,stich2018sparsified} as a special case, but operates at the level of local VJPs within arbitrary architectures and allows us to apply the approximation in a cascade of layers of a deep neural network.

Concretely, at a node $y=f(x)$ in the computational DAG with Jacobian $J_f(x)$, during backpropagation we replace the exact VJP $J_f(x)^\top g(x)$ by $\widehat J_f(x)^\top g(x)$ where $\widehat J_f(x)$ is a \emph{randomized} operator such that $\mathbb{E}[\widehat J_f(x)|g(x)]=J_f(x)$.
We analyze how the variance due to such local substitutions propagates along the reverse pass and identify randomized operators which yield favorable  
accuracy--efficiency trade-offs.
Specifically we propose to approximate operators $J_f$ by random rescaled low-rank  projections --either along data-dependent directions or along  fixed directions-- the latter leading to  diagonal mask-and-rescale schemes that keep a few coordinates at random.
Both preserve unbiasedness and let us target high-impact directions while meeting a budget on sketch size.

\paragraph{Contributions.}
\textbf{(1)} We propose replacing exact VJPs in a DAG by unbiased randomized estimators at selected nodes to reduce the cost of backpropagation at the expense of limited additional (stochastic gradient) variance.  
\textbf{(2)} We characterize the additional variance caused by such randomized VJPs along the backpropagation and identify regimes where errors dampen or amplify as a function of local Jacobian spectra. 
\textbf{(3)} We identify minimal variance unbiased VJPs for (i) general rank constraints and (ii) rank constraints for diagonal mask-and-rescale.
\textbf{(4)} We compare the training accuracy versus cost reduction trade-offs obtained for a selection of candidate unbiased randomized VJPs via experiments\footnote{Our implementations and experimental framework will be made publicly available upon publication.} with 
BagNets \citep{brendel2019approximating}, Visual Transformers \citep{dosovitskiy2021imageworth16x16words} and standard MLPs architectures. This leads us to identify promising candidates, in particular the ``$\ell_1$''-score approach which combines high cost reduction with marginal impact on test accuracy for fixed number of training steps.


A discussion of related works is given after the Conclusion.

\section{Sketching Reverse-Mode Automatic Differentiation}\label{sec:sketching}

\subsection{Variance in Stochastic Gradient Descent}
We consider the minimization of an objective function:
\begin{equation}
    \label{eq:minimization_problem}
    \min F(\theta) = \mathbb{E}_{\xi}[f(\theta, \xi)],
\end{equation}
where $\xi$ denotes a random data sample, and $\theta \in \mathbb{R}^d$ a set of parameters, optimization being done by Stochastic Gradient Descent (SGD) with minibatches \cite{bottou1991stochastic}.
At each iteration $t$, a minibatch of size $B$ is sampled and the update then writes:
\begin{equation}
    \label{eq:minibatch_sgd}
    \theta_{t+1} = \theta_t - \eta g_t, \quad g_t := \frac{1}{B} \sum_{b\in[B]} \nabla_\theta f(\theta_t, \xi_b).
\end{equation}

Under classical assumptions, e.g. $F$ $\beta$-smooth, non-negative, and variance bounded by $\mathbb{E}[\| g_t - \nabla F(\theta_t) \|^2 | \theta_t] \le \sigma^2$, canonical results show that SGD converges at a rate proportional to said variance.
In particular, with appropriate step-size $\eta$, one obtains a bound on the sizes of gradients after $T$ iterations
\begin{equation}
    \label{eq:sgd_rate}
    \min_{0 \le t \le T-1} \mathbb{E}[\|\nabla F_t\|^2] \le \frac{2(F_0 - F^*)}{\eta T} + \beta \eta \sigma^2,
\end{equation}
where $F^\star$ denotes the global minimum value of $F$ on $\mathbb{R}^d$.

Crucially, nothing in the proof depends on the specific scheme beyond unbiasedness and bounded variance. The bound is conservative, but it mirrors training behavior.

\subsection{Stochastic Gradient Surrogate}
\label{subsec:noisy_grads}
Suppose that now, instead of $g_t$, we use a modified gradient surrogate,
\begin{equation}
    \mathbb{E}[ \widehat{g}_t | g_t] = g_t ,
\end{equation}
with $\widehat{g}_t$ satisfying unbiasedness and bounded variance, \ie $\mathbb{E}[\|  \widehat{g}_t-g_t \|^2 |  g_t] \le V$. Then $\widehat{g}_t$ remains an unbiased estimator of $\nabla F(\theta_t)$, and standard SGD analyses apply unchanged.
In particular, the convergence bound in \eqref{eq:sgd_rate} still holds, up to replacing the variance $\sigma^2$ by $\sigma^2 + V$.
With this additional term, for a target accuracy $\varepsilon > 0$ on the right-hand side of \eqref{eq:sgd_rate} obtained under the optimal choice of step-size $\eta$, the required number of iterations scales as
\begin{equation}
\label{eq:num_iter}
    T = \mathcal{O}\left( \frac{(\sigma^2 + V) \beta (F_0 - F^\star)}{\varepsilon^2} \right).
\end{equation}
As a consequence, if we denote $\rho(V)$ the per-iteration computation induced by the injected variance $V$, a net computational gain (\ie wall-clock time) is achieved if
\begin{equation}
\label{eq:clock_gain}
    \rho(V)(\sigma^2 + V) \le \rho(0)\sigma^2.
\end{equation}
This condition reflects the variance-efficiency trade-off.
Thus, any additional noise that is conditionally unbiased only affects optimization through its contribution to variance.
This motivates our approach: we trade exact gradient computations for computationally cheap approximations, at the cost of introducing additional unbiased noise.

\subsection{Gradient Computation on a DAG}
We now describe how such additional noise arises when approximating backpropagation.

\paragraph{Standard Gradient Computations through a Computational Graph.}
Consider a computational DAG with leaf inputs and intermediate variables connected by differentiable operations. For each node \(i\), the forward computation for a sample $\xi$ is
\begin{equation}
    x_j(\xi) \;=\; f_i\!\bigl(\theta_i,\,(x_i(\xi))_{i \to j}\bigr),
\end{equation}
where \(i \to j\) indicates that \(x_i\) is an input to the operation producing \(x_j\), and \(\theta_i\) are parameters local to node \(i\).
Given a minibatch $\{\xi_1, \dots, \xi_B\}$, forward pass is applied independently to all $\xi_b$.
The empirical loss (scalar objective to minimize) is:
\begin{equation}
    L = \frac{1}{B} \sum_{b=1}^B \ell(x_{\text{out}}^{(b)}),
\end{equation}
where $x_{\text{out}}$ is the final node.
Define reverse-mode sensitivities (adjoints) \(g_i^{(b)} \triangleq \frac{\partial}{\partial x_i}
 \ell(x_{\text{out}}^{(b)})\) 
 and \(h_i^{(b)} \triangleq \frac{\partial}{\partial \theta_i} \ell(x_{\text{out}}^{(b)})\).
For convenience, write the local Jacobians
\begin{equation}
    \begin{aligned}
    J_{ij}^{(b)} \;\triangleq\; \partial_{x_j} f_i\!\bigl(\theta_i,(x_k^{(b)})_{k \to i}\bigr)^\top
    \\
    \text{and} \quad J_{i}^{(b)} \;\triangleq\; \partial_{\theta_i} f_i\!\bigl(\theta_i,(x_k^{(b)})_{k \to i}\bigr)^\top\,,
    \end{aligned}
\end{equation}
where \((\cdot)^\top\) denotes the adjoint.

Standard backwards message passing, using chain rule, writes, for a sample $b$
\begin{equation}
\label{eq:backprop_sample}
\left\{
\begin{aligned}
    g_i^{(b)} \;&=\; \sum_{j:\, i \to j} J_{ij}^{(b)}\, g_j^{(b)}, \\
    h_i^{(b)} \;&=\; J_{i}^{(b)} g_i^{(b)},
\end{aligned}
\right.
\end{equation}
where the recursion is seeded at the outputs with \(g_{\text{out}^{(b)}} = \frac{\partial}{\partial x_{\text{out}}} \ell(x_{\text{out}}^{(b)})\). 
Standard SGD would then use $h_i:=B^{-1}\sum_{b=1}^B h_i^{(b)}$ as the stochastic gradient for $\theta_i$. We shall also need the notation $g_i:=B^{-1}\sum_{b=1}^B g_i^{(b)}$.


\subsection{Gradient Estimation on a DAG}
We now study how layerwise variance propagates through the network. To this end, we introduce unbiased local Jacobian estimators $\widehat J_{ij}^{(b)}$ and propagate them through the layer cascade via
\begin{equation}
\left\{
\begin{aligned}
    \widehat g_i^{(b)} \;&=\; \sum_{j:\, i \to j} \widehat J_{ij}^{(b)}\,\widehat g_j^{(b)}, \\
    \widehat h_i^{(b)} \;&=\; \widehat J_{i}^{(b)}\, \widehat g_i^{(b)}, 
\end{aligned}
\right.
\label{eq:approx-backprop}
\end{equation}

We denote by $\widehat{g}_i = \tfrac{1}{B}\sum_b \widehat{g}_i^{(b)}$,  $\widehat{h}_i = \tfrac{1}{B}\sum_b \widehat{h}_i^{(b)}$the batch-averaged approximate gradients.
We then require
\begin{assumption}[Unbiasedness of local VJPs]
\label{as:local_vjp_unbiased}
For every sample $b$, node $i$ and edge $i \to j$,
\begin{equation}\begin{array}{l}
    \mathbb{E}\big[\hat J^{(b)}_{ij} | (\hat g^{(b')}_{k})_{k : i\to k, b' \in B}, (\hat J^{(b')}_{ik})_{i\to k, k\ne j, b' \in B} \big] = J^{(b)}_{ij},
    \\
\mathbb{E}\big[\hat{J}^{(b)}_i|\hat{g}^{(b)}_i]=J^{(b)}_i.
\end{array}
\end{equation}
\end{assumption}
We then have (see Appendix for a proof):
\begin{proposition}
    \label{prop:variance_backprop}
    Assume the seed at the output node is exact (\ie $\widehat{g}_{\text{out}} = g_{\text{out}}$) and that \cref{as:local_vjp_unbiased} holds.
    Then, for every node $i$ of the DAG,
    \begin{enumerate}
        \item[(i)] \emph{(Unbiasedness.)} $\;\mathbb{E}[\widehat g_i|g_i] \;=\; g_i$,  $\;\mathbb{E}[\widehat h_i|h_i] \;=\; h_i.$

        \item[(ii)] \emph{(Variance propagation.)}
        
        \begin{equation}
        \hspace*{-\leftmargin}%
        \begin{aligned}
        \mathbb{E} \big[\|\widehat{g_i} - &g_i \|^2 \big]
        = \sum_{j:i \to j} \mathbb{E}\!\left[\left\| \frac{1}{B}\sum_{b=1}^B
        (\widehat J_{ij}^{(b)} - J_{ij}^{(b)}) \widehat g_j^{(b)} \right\|^2 \right] \\
        &+ \mathbb{E}\!\left[\left\|\frac{1}{B} \sum_{b=1}^B\sum_{j:i \to j} \!
        J^{(b)}_{ij} (\widehat g^{(b)}_j - g^{(b)}_j) \right\|^2 \right].
        \end{aligned}
        \end{equation}

    \end{enumerate}
\end{proposition}

The
variance at a node $i$ is thus composed of two terms: 
The first is the variance induced locally from the chosen local VJP approximation, while the second  corresponds to the variance stemming  from approximations done at the preceding nodes in backpropagation.
The first term in this decomposition motivates criteria for the choice of VJP approximation discussed in the \cref{sec:mask}. 

The second term in this decomposition informs whether added variance dampens or increases during backpropagation. Indeed, this second term can be upper-bounded, using Cauchy-Schwarz inequality, by
$$
\frac{|\{j:i\to j\}|}{B}\sum_{b=1}^B\sum_{j:i\to j}\mathbb{E} \|J^{(b)}_{ij}\|^2 \|
\widehat g^{(b)}_j - g^{(b)}_j\|^2,
$$
where $\|J^{(b)}_{ij}\|$ denotes operator norm. If these operator norms are sufficiently small, the approximation errors will thus dampen along backpropagation.

\section{Randomized Vector-Jacobian Products in Linear Settings} \label{sec:mask}
We now study the construction of estimators such that
\[
\widehat J_{ij}^{(b)}\,\widehat g_{j}^{(b)} \;\approx\; J_{ij}^{(b)}\, g_{j}^{(b)}.
\]
In the context of VJPs, a natural approach is to introduce a (random) sketching matrix, as defined below.

\paragraph{Randomized Sketching Framework.} In the proposed setting, for and edge $i \to j$, with $J_{ij} \in \mathbb{R}^{m \times n}$ and $\widehat g_j\in \mathbb{R}^n$, we consider a \emph{random sketching matrix} $R\in \mathbb{R}^{n \times n}$ such that $\mathbb{E}[R|\mathcal{F}] = I_n$  where $\mathcal{F}$ captures information relative to all other VJP approximations and batch sampling.  Then
\begin{equation}
   \forall b \in {1, ..., B},\, \widehat J_{ij}^{(b)} \,  = J_{ij}^{(b)} \, R,
\end{equation}so the estimator is locally unbiased, as in \cref{as:local_vjp_unbiased}. As motivated in \cref{subsec:noisy_grads}, we aim at lowering additional variance, \ie we aim at finding the sketch that minimizes the distortion, measured by $\mathsf{L}^2$ cost\footnote{Strictly speaking we consider conditional expectations with respect to $\mathcal {F}$, the $\sigma$-field capturing information relative to all other approximations and batch sampling. We do not make this explicit to lighten notation.}
\begin{align}
\label{eq:l2_cost}
    &\mathcal{L}(R) \,\triangleq\, \frac{1}{B} \, \sum_{b=1}^B \,  \mathbb{E}[\|J_{ij}^{(b)} \widehat g_j^{(b)} - J_{ij}^{(b)} R \, \widehat g_j^{(b)} \|^2],\\
&= \frac{1}{B}\sum_{b=1}^B
\mathbb{E}\,\hbox{Tr}\left[
 J_{ij}^{(b)\top}J_{ij}^{(b)}(I-R)\widehat g_j^{(b)}\widehat g_j^{(b)\top}(I-R)^\top
\right]. \notag
\end{align}
This yields a quadratic objective in the compressed surrogate $R$, which preserves unbiasedness of the estimator for $J$ without adding variance. In the batch-wise regime, $R$ is shared across all $b \in \mathcal{B}$; an exact solution follows from the sketching lemma below.



\if False
\paragraph{Linear nodes}
\killian{J'imagine qu'il faut en parler avant les proposition/corollaire, puisque je m'en sers dans la preuve ?}
Then, the Jacobian \wrt activations is the same for every element of the batch, \ie $\forall b \in \mathcal{B}, J_{ij}^{(b)} = J_{ij} = W^\top$ and \eqref{eq:l2_cost} becomes

\begin{equation}
    \mathcal{L}(R) \,=\, \frac{1}{B} \, \sum_{b=1}^B \,  \mathbb{E}[\|J_{ij} \widehat g_j^{(b)} - J_{ij} R \, \widehat g_j^{(b)} \|^2].
\end{equation}
\fi

\paragraph{Optimal Unbiased Low-Rank Matrix Approximation.}

We first give a general result, which we will subsequently apply to solve the previous problem in specific settings:

\begin{lemma}[Optimal unbiased random sketch under rank constraint]
\label{thm:svd_sketch}
    Let $M$ be a fixed matrix in $\dR^{m\times n}$. Let $q=\min(m,n)$ and denote the SVD of $M$ by $M=\sum_{i=1}^q \sigma_i u_i v_i^\top$, where $\sigma_i$, $u_i$, $v_i$ are respectively its singular values, left and right singular vectors. Then among matrices $S$ of rank at most $r$ where $r<q$, such that $\mathbb{E}[S]=M$, the choice that minimizes the error in squared Frobenius norm $\dE\|M-S\|^2_F$ is obtained by letting 
    $$
    S=\sum_{i=1}^q \sigma_i \frac{Z_i}{p_i}u_i v_i^\top$$
    where the $Z_i$ are correlated Bernoulli random variables with respective parameters $p_i$. The $Z_i$ are such that $\sum_{i=1}^q Z_i\equiv r$, and the $\{p_i\}_{i\in [q]}$ are minimizers of $\sum_{i=1}^q \sigma^2_i/p_i$ among weights $p_i\in [0,1]$ that sum to $r$.


\end{lemma}

The proof is detailed in the Appendix; it proceeds by first establishing a lower bound on the expected squared Frobenius norm of any matrix $S$ such that $\mathbb{E}[S]=M$, using Jensen's inequality together with convexity properties of functions of singular values, combined with the so-called weak majorization theory. It then shows that this lower bound is reached by the specific choice described above. The procedure to sample correlated Bernoulli random variables with fixed sum is also detailed.

We recently found out that \citet{barnes2025unbiasedlowrankapproximationminimum} have independently shown \cref{thm:svd_sketch}, establishing the upper-bound and noticing that the lower-bound followed from the suppementary material in \citet{benzing2019optimalkroneckersumapproximationreal}. Our proof is however different and may thus be of independent interest.

\paragraph{Application to VJPs.}


The above Lemma has the following direct application:

\begin{lemma}[Distortion in Linear Nodes]
\label{lem:linear_distortion}
    Consider $i$ a Linear node,
    \begin{equation}
    \nonumber
        x_j(\xi) =  W_{ij} x_i(\xi) + \mathbf{b}, \quad W_{ij} \in \mathbb{R}^{d_{\text{out}} \times d_{\text{in}}}, \mathbf{b} \in \mathbb{R}^{d_{\text{out}}},
    \end{equation}
    and $g_i^{(b)}, b \in [\![1, B]\!]$ the gradients of a considered batch $\mathcal{B}$.
    Then the Jacobian \wrt activations is the same for element of the batch, and, denoting
    \begin{equation}
        \label{def:g_matrix}
        G := \big[ g^{(1)} \dots g^{(B)}\big] \in \mathbb{R}^{d_{\text{out}} \times B}
    \end{equation}
    the matrix whose columns are the backpropagated gradients, the distortion defined in \eqref{eq:l2_cost} becomes
    \begin{equation}
        \mathbb{E}\big[ \Tr \big(J^\top J (I-R) \Gamma_\mathcal{B} (I - R)^\top \big) \big],
    \end{equation}
    where $\Gamma_{\mathcal{B}} := \frac{1}{B} GG^\top$ denotes the empirical second moment matrix of the batch, and $J=W_{ij}$. 
\end{lemma}

\begin{proof}

    Expanding the expression of $\mathcal{L}$ and using the circularity of the trace, we obtain for a single batch $g^{(b)}$
    \begin{equation}
        \nonumber
        \mathbb{E}\|J(I-R) g^{(b)} \|^2 = \mathbb{E}\Tr \big(J^\top J (I-R) g^{(b)} g^{(b) \top} (I - R)^\top \big).
    \end{equation}
    Now, using the linearity of trace and expectancy, averaging over the batch yields
    \begin{equation}
        \nonumber
        \mathcal{L}(R)\! = \mathbb{E}[\Tr\! \big(J^\top J (I-R) (\frac{1}{B} \sum_{b=1}^B  g^{(b)}\! g^{(b) \top}) (I - R)^\top \big)],
    \end{equation}
    hence the result.
\end{proof}
This, combined with the previous Lemma, yields
\begin{proposition}[Minimal Distortion rank $r$ Unbiased Sketch]
\label{prop:optimal-column-sketch}
Let $J\in\mathbb{R}^{m\times n}$ and $\Gamma_{\mathcal{B}}$ defined as above. 
Define the symmetric matrix
\begin{equation}
    \label{eq:basis_svd}
    \Gamma_{\mathcal{B}}^{1/2} J^\top J \Gamma_{\mathcal{B}}^{1/2} \;=\; U\,\Sigma\,U^\top,
\end{equation}
where $U$ is orthogonal and $\Sigma=\mathrm{diag}(\sigma^2_1,\ldots,\sigma^2_n)$ with $\sigma_1\ge \cdots \ge \sigma_n\ge 0$.
Among all random matrices $R$ with $\mathbb{E}[R]=I_n$ and
whose 
rank is bounded by $r$,
a minimizer of the $\mathsf{L}^2$ cost \Cref{eq:l2_cost} is attained by sketches
that are diagonal in the eigenbasis of $\Gamma_{\mathcal{B}}^{1/2}J^\top J\Gamma_{\mathcal{B}}^{1/2}$:
\begin{equation}
    \begin{aligned}
        R^* &= \Gamma_{\mathcal{B}}^{1/2}\,U\,B\,U^\top\,\Gamma_{\mathcal{B}}^{-1/2} \\
        B &= \mathrm{diag}(\frac{z_1}{p^*_1},\dots,\frac{z_n}{p^*_n}), \text{ with } z_i \sim \mathcal{B}(p^*_i)
    \end{aligned}
    \label{eq:opt-R-form}
\end{equation}
with $\mathcal{B}(p)$ being the Bernoulli distribution of parameter $p$, and with selection probabilities $\{p^*_i\}_{i=1}^n$ minimizing $\sum_{i=1}^n \sigma^2_i/p_i$ under the constraint $\sum_i p_i=r$.

%
\end{proposition}

\begin{proof}[Sketch of proof]
    Expanding the quadratic form obtained in \cref{lem:linear_distortion} and using $\mathbb{E}[R]=I_n$ gives
    \begin{equation}
        \mathcal{L}(R) = -\Tr(J\Gamma_{\mathcal B}J^\top) \;+\; \mathbb{E}\left[ \Tr \big(J R\Gamma_{\mathcal B}R^\top J^\top\big)\right],
    \end{equation}
    so minimizing it reduces to minimizing the second term. Writing $S=J R \Gamma_{\mathcal{B}}^{1/2} $, we are thus looking for a matrix $S$ such that $\mathbb{E}[S]=M=J\Gamma_{\mathcal{B}}^{1/2}$, of rank at most $r$, and with minimal expected squared Frobenius norm. Lemma \ref{thm:svd_sketch} gives us the optimal choice for $S$, from which we can deduce the optimal choice for $R$.
  Details are given in the Appendix.
\end{proof}

\paragraph{Diagonal Sketches.}
In many neural network layers, restricting sketches to diagonal operators leads to significant implementation and memory advantages hence the following analysis of diagonal sketches (\eg fully masking coordinates) under the same budget constraints.

\begin{lemma}[Diagonal mask with expected size at most $r$]
\label{cor:diag-mask}
We now restrict $R$ to be diagonal, $R=\mathrm{diag}(r_1,\ldots,r_n)$ with
\[
r_i \;=\; z_i/p_i, \quad
z_i \sim \mathcal{B}(p_i)\, \quad \text{independent,}\quad
p_i \in ]0,1],
\]
so that $\mathbb{E}[R]=I_n$.  For such diagonal mask, under the expected rank constraint $\sum_{i=1}^n p_i\le r$, 
the minimal $\mathsf{L}^2$ distortion \Cref{eq:l2_cost} is obtained be choosing probabilities $p_i$ that minimize 
$
\inf_{p}\sum_{i=1}^n \frac{a_i}{p_i} $ under the constraint $\sum_i p_i\le r$, where 
$
a_i \;:=\; (\Gamma_{\mathcal{B}})_{ii}\,(J^\top J)_{ii}$.
\end{lemma}
The details are provided in the Appendix. Note that here we have considered independent Bernoulli random variables, and have thus imposed a constraint on the  rank in expectation rather than almost surely. This modification is required for the proof to go through. 
\section{Jacobian Approximations} \label{sec:practical_masks}
In this section, we look over different practical Jacobian Approximation methods 
that elaborate on those in the previous section.
\subsection{First Strategy: Applying Uniform Masks} 
\label{subsec:masking_operators}
We first discuss simple uniform masks, common in the literature, which serve as baselines for approximating vector–Jacobian products.
We denote them as $M$ (as opposed to geometrically informed sketches $R$), and their goal is still to approximate VJPs, \ie $JMg \approx Jg$.

\paragraph{Per-Element Masks.}
Per-element masking is common in practical deep leaning settings \citep{srivastava2014dropout,lee2019mixout}, and is the first considered approximation.
Let $p \in (0, 1]$ a parameter so that $M_{kl} = B_{kl}/p$, $B_{kl}\sim\mathcal{B}(p)$, and $\mathbb{E}[M_{kl}]=1$.
Element-wise application $\widehat J \;=\; M \odot J$ gives an unbiased estimator of the Jacobian.
It is trivial to verify that $\mathbb{E}[\widehat J] \;=\; J$, hence unbiasedness. 
For such masks, computational load reduction may stem from the element-wise sparsity of the approximating matrix rather than from a   reduced rank property.



\paragraph{Per-Column Masks.}
A second simple strategy is to apply independent Bernoulli gating per column (or channel) of $J$.
For a certain $p \in (0, 1]$, draw $z_i \stackrel{\text{i.i.d.}}{\sim} \mathcal{B}(p)$ and define an unbiased and expected rank $r$ mask as
\begin{equation}
    M \;=\; \operatorname{diag} \big(z_1/p, \dots, z_n/p \big).
\end{equation}


\paragraph{Per-Sample Masks.}
The last investigated masking method, inspired by \citet{woo2024dropbp}, consists in applying a Bernoulli gate to the whole product at sample level.
For a fixed $p \in (0, 1]$, draw a scalar gate $z \sim \mathcal{B}(p)$ and set
\begin{equation}
    M = (z/p) I_n.
\end{equation}
This $M$ also satisfies unbiasedness and expected low-rank.


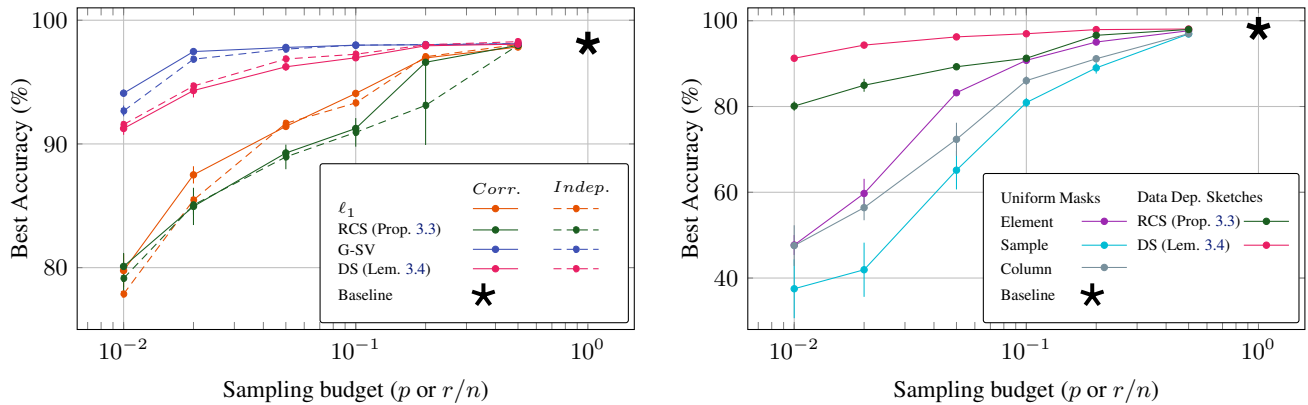
\begin{figure*}[t]
    \centering

    \begin{subfigure}[t]{0.48\textwidth}
        \centering
        \input{figures_tex/corr_vs_no_corr_2}
        \caption{Impact of correlation in Bernoulli sampling}
        \label{fig:corr_vs_no_corr}
    \end{subfigure}
    \hfill
    \begin{subfigure}[t]{0.48\textwidth}
        \centering
        \input{figures_tex/uninformed_vs_score}
        \caption{Comparison between \emph{Masking} and \emph{Sketching} Methods}
        \label{fig:uninformed_vs_score}
    \end{subfigure}

    \caption{Comparison of sampling strategies and scoring methods.}
    \label{fig:correlation_and_weighted}
\end{figure*}

\subsection{Second Strategy: Data Dependent Sketches}
\label{subsec:sketching_operators}
We now turn to a class of \emph{Data Dependent} sketching operators that allocate the sketching budget according to data-dependent importance weights.
Recall from \cref{sec:mask} that under unbiasedness and rank constraints, minimizing the distortion of a VJP approximation reduces to a convex optimization problem over the sampling probabilities:
\begin{equation}
    \label{eq:convex_program}
    \min \; \sum_{i=1}^n \frac{w_i}{p_i}, \quad \st \sum_{i=1}^n p_i \le\; r, \; p_i \in (0, 1].
\end{equation}
Here $p_i$ denote the probabilities defining the sketch, while the weights $w_i$ quantify the relative importance of each direction as induced by the considered sketching constraints.
We hereafter describe some possible $w_i$ choices leading to sketching operators.

\paragraph{Rank-Constraint Sketches (RCS).}
We first consider the construction in \cref{prop:optimal-column-sketch}.
With solely unbiasedness and rank restrictions imposed to $R$, the importance weights $w_i$ correspond to the squared singular values $\sigma_i^2$ of the matrix $ \Gamma_\mathcal{B}^{1/2} J^\top J \Gamma_\mathcal{B}^{1/2}$, \ie $p_i^\star$ are obtained by solving \eqref{eq:convex_program} with $w_i = \sigma_i^2$, and $R$ is sampled as a diagonal matrix of Bernoulli entries, parameterized by $p_i^\star$.
The resulting sketching operator uses these probabilities to sample directions in the eigenbasis of $\Gamma_\mathcal{B}^{1/2} J^\top J \Gamma_\mathcal{B}^{1/2}$ (see $R^\star$ construction in \cref{prop:optimal-column-sketch}).

\paragraph{Diagonal Sketches (DS).} 
A variant is obtained by restricting the sketching operator $R$ to be diagonal, as studied in \cref{cor:diag-mask}.
Under this constraint, \eqref{eq:convex_program} admits importance weights of the form $w_i = (\Gamma_\mathcal{B})_{ii} (J^\top J)_{ii}$.

This estimator generalizes per column masking by allowing non-uniform, data-dependent sampling probabilities.

\paragraph{Alternative Weights Proxies.}
Beyond these two theoretically grounded constructions, \eqref{eq:convex_program} provides a more general design principle.
Once a set of non-negative importance weights $w_i$ is specified, the resolution gives an unbiased sketch under a budget constraint, by sampling with probabilities proportional to $\sqrt{w_i}$.

Therefore, we explore alternative, computationally cheaper proxies for the weights $w_i$ that do not strictly correspond to the quantities derived in \cref{sec:mask}.
In particular, we consider $\ell_1$-, $\ell_2$-norms, and empirical variances of gradient coordinates as simple measures of magnitude and dispersion; we denote the corresponding proxy-based sketch $\ell_1$, $\ell_2$ and $\Var$.
We also consider importance weights based on the singular value decomposition of the batch gradient matrix $G$ (as defined in Eq. \eqref{def:g_matrix}), and refer to this strategy as $G$-Singular Values (G-SV).
Finally, since the optimal probabilities scale with $\sqrt{w_i}$, we additionally evaluate the squared versions of these proxies.

\section{Numerical Analysis} \label{sec:numerical_analysis}

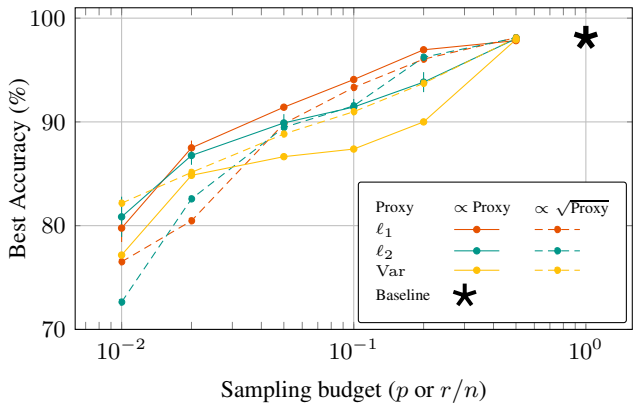
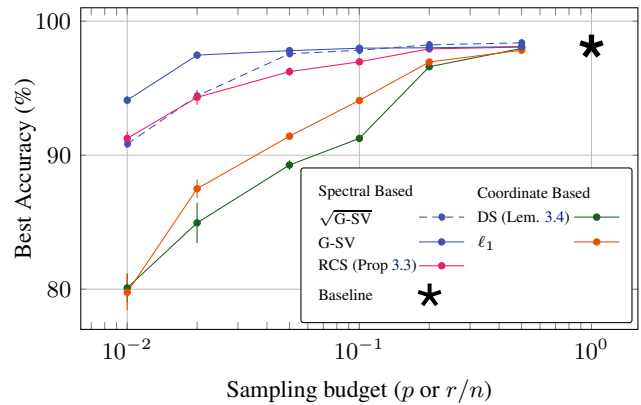
\begin{figure*}[t]
\captionsetup[subfigure]{skip=2pt}
    \centering
    \begin{subfigure}[t]{0.48\textwidth}
        \centering
        \input{figures_tex/scores_comparison}
        \caption{Comparison of simple proxies}
        \label{fig:scores_comp}
    \end{subfigure}
    \hfill
    \begin{subfigure}[t]{0.48\textwidth}
        \centering
        \input{figures_tex/svd_vs_no_svd}
        \caption{Comparison of spectral and coordinates based methods}
        \label{fig:svd_vs_no_svd}
    \end{subfigure}

    \caption{Comparison of types of weighted methods.}
    \label{fig:scores_all}
\end{figure*}

\paragraph{Experimental Setting.}
We first train 4-layer MLPs on MNIST \citep{deng2012mnist}: input dimension 784, two hidden layers of width \(64\), and a \(10\)-way output.
Training uses SGD without momentum nor a learning-rate schedule, for 50 epochs with cross-entropy loss, and gradient clipping at norm \(1\).
For each seed, the learning rate is cross-validated over the grid $\{10^{-0.25i} \mid i \in [\![0,12]\!]\}$, and we report results for the best-performing value.
In the figures, we approximate VJPs at all layers, except for the baseline, obtained with a standard training without any VJP approximation.





\paragraph{Correlation in Sampling.}
As discussed in \cref{thm:svd_sketch}, directions can be sampled either coordinate-wise independently, or by using correlated sampling schemes to ensure a fixed size rank.
This latter construction reduces variability in budget allocation across iterations.
In practice, we observe that enforcing this rank constraint leads to slightly improved performance, particularly in low-budget regimes, although the effect remains relatively modest.
Results are reported in \cref{fig:corr_vs_no_corr}, and we adopt the correlated variant as the default for the subsequent discussions.


\paragraph{Uniform Masks.}
We now compare methods with one another and start by looking at uniform masking methods in opposition to data-dependent sketches (see Section  \ref{sec:practical_masks}).
As the former are agnostic to local VJP's characteristics, they provide natural baselines that isolate the effect of injecting noise.
\cref{fig:uninformed_vs_score} shows that the latter consistently outperform the three agnostic methods, indicating that leveraging local information leads to visible empirical gains.


\paragraph{Weight Proxies Comparison.}
In \cref{sec:practical_masks}, importance weights-based methods were introduced.
\cref{fig:scores_comp} reports a comparison between different weights proxies strategies.
All methods yield very similar accuracy curves, indicating that no proxy incontestably dominates the others.
Nonetheless, $\ell_1$-based probabilities lie on the upper envelope of the curves.
Although the differences remain limited, this trend is consistent, and, therefore, $\ell_1$-based sampling is therefore used as a default choice in the subsequent experiments.

\paragraph{Coordinate- VS Spectral-Based Strategies.}
We now consider strategies that rely on SVD, namely Rank Constrained Sketching and the approach based on singular values of $G$ and compare them to methods grounded in the canonic coordinate space.
These approaches are computationally more expensive than aforementioned proxies, and are hence expected to reach better accuracy.
\Cref{fig:svd_vs_no_svd} confirms this intuition: spectral methods consistently outperform simpler proxies across the range of sampling budgets.
G-SV strategy yields better results than its square root counterpart and is retained in the following experiments.


\if False
\paragraph{Training Loss.}
\killian{Cette partie (avec le tableau) c'est peut être un peu se tirer une balle dans le pied, faut-il laisser ?}
\cref{tab:loss_isoflops} reports the final training loss achieved by the different approximation methods, compared to the baseline, under a normalized computational budget.

To account for the reduced cost of approximate VJPs, the number of training iterations is scaled according to the effective per-iteration computational ratio.
Concretely, if all layers are approximated with a ratio $p$, then one
baseline iteration is considered equivalent to $1/p$ iterations of the  corresponding method.
We expose here results for $p=0.5$ and $p=0.05$.
SVD based methods outperform the baseline in both configurations, while other methods are struggling in the most aggressive approximation configuration ($p=0.05$).

\begin{table}[h]
\centering
\small
\setlength{\tabcolsep}{4pt}
\begin{tabular}{l c c l c c}  
\hline
\textbf{Model} & \textbf{Scaling} & \textbf{Iteration} & \textbf{Method} & \multicolumn{2}{c}{\textbf{Train Loss}} \\
\hline
\multirow{12}{*}{MLP}      & 1                       & 100                     & Baseline & 1.027 \\
\hdashline
                          & \multirow{6}{*}{2} & \multirow{6}{*}{200} & EW. Masking & 1.080 \\
                          &                    &                      & Partial Participation & 0.892 \\
                          &                    &                      & $\ell_1$ & 0.795 \\
                          &                    &                      & G-Singular Values &  0.640 \\
                          &                    &                      & \cref{prop:optimal-column-sketch} & 0.693 \\
                          &                    &                      & \cref{cor:diag-mask}&  0.729 \\
\hdashline
                          & \multirow{6}{*}{20}& \multirow{6}{*}{2000}& EW. Masking & 2.276 \\
                          &                    &                      & Partial Participation & 2.262 \\
                          &                    &                      & $\ell_1$ & 1.205 \\
                          &                    &                      & G-Singular Values &  0.386 \\
                          &                    &                      & \cref{prop:optimal-column-sketch} & 0.569 \\
                          &                    &                      & \cref{cor:diag-mask}&  2.246 \\
\hline
\hline
\end{tabular}
\caption{Train loss for different methods.}
\label{tab:loss_isoflops}
\end{table}
\fi

\paragraph{BagNet and ViT.}

To assess scalability of our approaches, we conduct experiments with the CIFAR-10 dataset \citep{cifar10}, and larger architectures, namely BagNet-17 \citep{brendel2019approximating}, and Visual Transformers (ViT) \citep{dosovitskiy2021imageworth16x16words} of comparable sizes.
The former is conceptually and structurally close to ResNets \citet{he2015deepresiduallearningimage}, the key difference being that BagNet relies mainly on $1 \times 1$ convolutions, which we assimilate as linear layers and sketch.
The latter is a transformer, and thus consists in attention blocks followed by linear feed-forward layers, to which we apply sketching operators as well.

For both architectures, we established a baseline, and then applied our VJP approximation methods to all linear layers, except the output classification layer.
These approximations were performed for budgets $p \in \{0.05, 0.1, 0.2, 0.5\}$, with learning rates cross-validated over five logarithmically spaced values around the baseline setting.
Appendix provides training and ViT-specific architectural details.

\paragraph{Discussion:}
\cref{fig:bagnet_and_vit} reports the results obtained for six retained methods on each architecture.
Overall, these results are very encouraging, with Diagonal Sketching consistently emerging as a particularly strong choice.
On both architectures, the accuracy degradation of our approaches remains limited, even for relatively small values of the sampling parameter.
Furthermore, our plots clearly highlight the benefits of using data- and/or Jacobian-dependent strategies over uniform masking. 

Concerning RCS, although we exposed its optimality for local distortion, it does not systematically dominate simpler strategies in practice.
We conjecture that this is due to the local nature of the approximation: sketches are constructed independently at each VJP, without accounting for dynamics through subsequent layers.

\begin{figure*}[t]
    \centering

    \begin{subfigure}[t]{0.48\textwidth}
        \centering
        \input{figures_tex/all_methods_bagnet}
        \caption{Jacobian Approximation on BagNet-17}
        \label{fig:bagnet_comparison}
    \end{subfigure}
    \hfill
    \begin{subfigure}[t]{0.48\textwidth}
        \centering
        \input{figures_tex/all_methods_vit}
        \caption{Jacobian Approximation on ViT}
        \label{fig:vit_comparison}
    \end{subfigure}

    \caption{Sketching on larger architectures.}
    \label{fig:bagnet_and_vit}
\end{figure*}
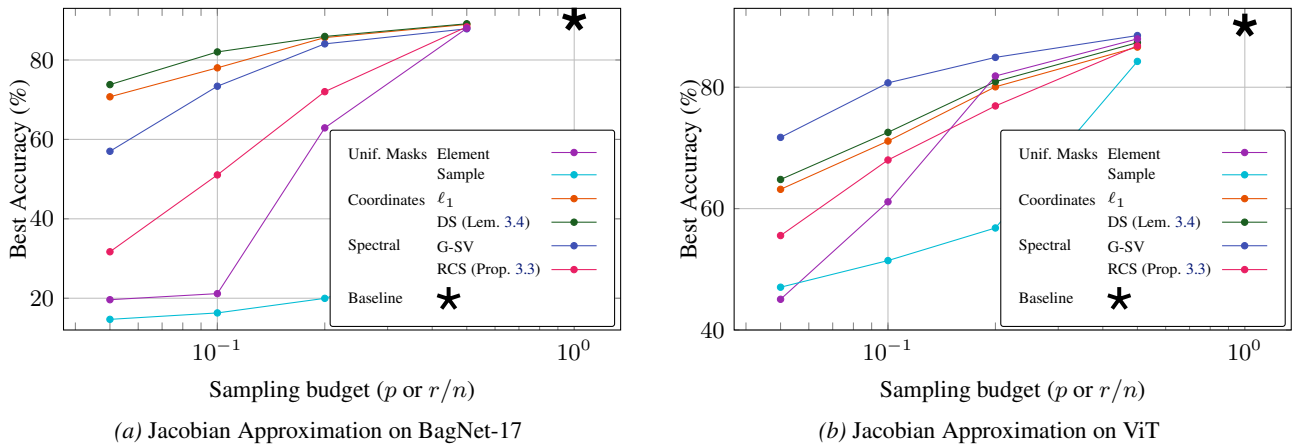
\section{Conclusion \& Future Work} \label{sec:conclusion}
In this work we initiated the study of noisy unbiased backpropagation in a computational DAG based on approximate VJPs. The analysis of how variance evolves during backpropagation led to novel design principles for unbiased VJPs under budget constraints. 

We then derived several VJP approximation techniques elaborating on these principles
and conducted an extensive empirical comparison on MLP architectures.
After establishing their effectiveness in this controlled setting, we further applied the proposed methods to larger models, where they similarly exhibit solid results and confirm the scalability of our approach.

Our results already show strong performance without any cross-layer coordination between sketches.
A natural next step is to design coordinated sketching strategies that account for variance propagation through the DAG and potentially improve general performance.
From a more practical point of view, estimating costly statistics intermittently rather than at each step, and/or jointly adapting optimization hyperparamters and sketching policies are promising future work directions.

\section{Related Work} \label{sec:related}

\paragraph{Weight Gradient Compression.}
Gradient compression in distributed and federated learning primarily acts \emph{after} backpropagation, directly on parameter gradients:
\[
\widehat h_i \;=\; \mathcal C(h_i),
\qquad \mathbb{E}[\mathcal C(h_i)\mid h_i]=h_i,
\]
covering classical unbiased compressors and their analyses \citep{wang2023cocktailsgd,konevcny2016federated,horvath2020better}.
When the compressor is (or is approximated by) a random linear map, $\mathcal C(v)=S_i v$, this becomes exactly our parameter-side sketch
$\widehat J_i = S_i J_i$.
Sparsification \citep{stich2018sparsified} corresponds to diagonal $S_i$ (mask-and-rescale), and error-feedback schemes \citep{richtarik2021ef21}
can be interpreted as adding a stateful correction term to compensate bias induced by repeatedly applying $S_i$.
A key distinction with our focus is \emph{where} the randomness enters:
weight-gradient compressors operate on the \emph{final} gradient signal $h_i$ and are largely agnostic to the internal Jacobian factorization,
whereas we target the \emph{intermediate} VJPs to trade compute/activation-traffic for variance. Note also that a parallel line of work seeks to overlap computation with weight gradients communication \cite{nabli2025acco}, that we could combine with the results of this paper.

\paragraph{Backward-pass Compression.}
Coordinate masking methods such as meProp \citep{sun2017meprop} correspond to diagonal sketches $R_{ij}$ that keep a subset of coordinates of $\widehat g_j$, in the line of sketches analyzed in Lemma~\ref{cor:diag-mask} (although we enforce unbiasedness).
More structured policies, e.g.\ skipping Jacobian applications on selected layers, corresponding to masking on the chosen edges,
which covers layer dropping schemes reported for MLP-like architectures \citep{fagnou2025accelerated,neiterman2024layerdropback}.
Pipeline settings that compress backward activations across devices also amount to picking $R_{ij}$ on inter-stage edges
(often diagonal masks or low-rank sketches), as in pipeline-oriented designs \citep{kong2025clapping,narayanan2019pipedream,huang2019gpipe}.
Batch-level masking of gradients \citep{bershatsky2022memory} is a special case where the same mask is applied to \emph{all} samples in a mini-batch.

\paragraph{Quantization and Randomized Arithmetic.}
Mixed precision and quantized training accelerate linear algebra in forward/backward passes \citep{micikevicius2017mixed,johnson2018rethinking,wiedemann2020dithered}.
Unbiased stochastic quantizers can be written as random maps $Q(v)$ with $\mathbb{E}[Q(v)\mid v]=v$; when $Q$ is coordinate-wise and unbiased,
it is equivalent to applying a random diagonal operator to $v$ in expectation, hence can be absorbed into our $R_{ij}$ acting on $\widehat g_j$.
In this sense, many “unbiased quantization” schemes are instances of our local-unbiased sketching assumption, differing mainly by the distributional form of $R_{ij}$.

\paragraph{Alternatives to Backprop.}
Methods such as direct feedback alignment or decoupled training replace the exact chain of Jacobians by synthetic or partially decoupled signals
\citep{nokland2016direct,rivaud2025petra,belilovsky2020decoupled,jaderberg2017decoupled}.
These approaches are sometimes framed as using \emph{different} operators $\widehat J_{ij}$ that are not of the form $J_{ij}R_{ij}$ with
$\mathbb{E}[R_{ij}]=I$; in particular, they generally do not preserve conditional unbiasedness of true gradients and may change the optimization objective.
By contrast, we keep the exact computational graph but randomize VJPs with unbiasedness constraints, so the only effect on SGD is via variance
(Section~\ref{subsec:noisy_grads}).

Unbiased alternatives that rely on forward-mode signals (JVPs) propagate random probes through the network and then form gradient estimators
\citep{ren2022scaling,fournier2023can}. At a high level, these can be viewed as constructing an unbiased estimator of the \emph{full-network} Jacobian action
using random sketches; however, the randomness is injected via forward-mode probes rather than via local VJP replacements.
Unbiased truncations for sequences \citep{tallec2017unbiased} can be seen as masking temporal edges (setting some Jacobian factors to zero) with appropriate
rescaling to keep expectations correct—again matching the “randomly drop factors but preserve unbiasedness” principle, though specialized to time-unrolled graphs.

A related but distinct direction applies masks in the \emph{forward} pass as well (hence changing the function being differentiated);
this can be written as replacing $J_{ij}$ itself by a different Jacobian $\widetilde J_{ij}$, which in general breaks our conditional-unbiasedness
requirement and can yield biased gradient surrogates (e.g.\ when the mask is not compensated in the backward pass) \citep{singh2025model}.
Our framework isolates the “safe” regime: preserve unbiasedness via $R_{ij}$  while controlling the induced variance.

\paragraph{Sketching for Linear Approximation.}
Independent of deep learning, masking and sketching are classical tools for efficient  sparse approximation \citep{kireeva2024randomized},
including kernel subsampling with compensation \citep{rudi2017falkon}, sketching for ridge regression \citep{gazagnadou2022ridgesketch},
and communication-efficient federated protocols based on sketches \citep{shrivastava2024sketching}.
Our contribution connects these sketching ideas to \emph{where} they appear in backpropagation: we study sketches as local unbiased replacements of VJPs
(i.e.\ choices of $R_{ij}$), and analyze how the induced variance propagates through the DAG (Proposition~\ref{prop:variance_backprop}),
thereby making the \emph{depth and placement} of sketching a first-order design choice rather than a purely systems-level consideration.

\section*{Acknowledgements}
EO acknowledges funding from PEPR IA (grant SHARP ANR-23-PEIA-0008). This work was granted access to the HPC resources of IDRIS under the allocation 2025-AD011015884R1 made by GENCI.
KB thanks coworkers Pierre Aguié, Loïck Chambon and Alessia Rigonat for insightful discussions and/or technical support. He acknowledges funding from PEPR IA (grant REDEEM ANR-23-PEIA-0005). This work was granted access to the HPC resources of IDRIS under the allocation AD011017074 made by GENCI.
LM acknowledges funding from PR[AI]RIE-PSAI – Paris School of Artificial Intelligence, reference ANR-23-IACL-0008.

The authors gratefully acknowledge Christophe Bouder for gracefully facilitating  access to SCAI GPU resources.

\section*{Impact Statement}

This paper presents work whose goal is to advance the field of Machine Learning. There are many potential societal consequences of our work, none which we feel must be specifically highlighted here.

\bibliography{unbiased_grads}
\bibliographystyle{icml2026}

\newpage
\appendix
\onecolumn

\section{Proofs}


\subsection{Proofs of \cref{sec:sketching}}

\begin{proposition}[Restatement of \cref{prop:variance_backprop}]
    Assume the seed at the output node is exact (\ie $\widehat{g}_{\text{out}} = g_{\text{out}}$) and that \cref{as:local_vjp_unbiased} holds.
    Then, for every node $i$ of the DAG,
    \begin{enumerate}
        \item[(i)] \emph{(Unbiasedness.)} $\;\mathbb{E}[\widehat g_i|g_i] \;=\; g_i$,  $\;\mathbb{E}[\widehat h_i|h_i] \;=\; h_i.$

        \item[(ii)] \emph{(Variance propagation.)}
        
        \begin{equation}
        \mathbb{E} \big[\|\widehat{g_i} - g_i \|^2 \big]
        = \sum_{j:i \to j} \mathbb{E}\!\left[\left\| \frac{1}{B}\sum_{b=1}^B
        (\widehat J_{ij}^{(b)} - J_{ij}^{(b)}) \widehat g_j^{(b)} \right\|^2 \right]
        + \mathbb{E}\!\left[\left\|\frac{1}{B} \sum_{b=1}^B\sum_{j:i \to j} \!
        J^{(b)}_{ij} (\widehat g^{(b)}_j - g^{(b)}_j) \right\|^2 \right]
        \end{equation}       
    \end{enumerate}
\end{proposition}

\begin{proof}
We proceed by backward induction over the DAG, starting from the output node.

\paragraph{Unbiasedness.}
For each sample $b$, the approximate backward recursion reads
\begin{equation}
    \widehat g^{(b)}_i = \sum_{j:i\to j} \widehat J^{(b)}_{ij}\,\widehat g^{(b)}_j,
\end{equation}
as stated in \eqref{eq:approx-backprop}.
Taking the conditional expectation given $(\widehat g^{(b)}_j)_{j:i\to j}$ and using \cref{as:local_vjp_unbiased} yields
\begin{equation}
    \mathbb{E}\big[\widehat g^{(b)}_i | (\widehat g^{(b)}_j)_j \big] = \sum_{j:i\to j} J^{(b)}_{ij}\,\widehat g^{(b)}_j .
\end{equation}
By the induction hypothesis, in reverse order from $g_{\text{out}}^{(b)}$, $\mathbb{E}[\widehat g^{(b)}_j]=g^{(b)}_j$, hence
\begin{equation}
\mathbb{E}[\widehat g^{(b)}_i] = \sum_{j:i\to j} J^{(b)}_{ij} g^{(b)}_j = g^{(b)}_i .
\end{equation}
Averaging over the batch gives $\mathbb{E}[\widehat g_i \mid g_i]=g_i$.

The same argument applies to $\widehat h^{(b)}_i=\widehat J^{(b)}_i\,\widehat g^{(b)}_i$, yielding $\mathbb{E}[\widehat h_i \mid h_i]=h_i$.

\paragraph{Variance propagation.}
For each sample $b$, we decompose
\begin{equation}
    \widehat g^{(b)}_i - g^{(b)}_i = \sum_{j:i\to j} (\widehat J^{(b)}_{ij}-J^{(b)}_{ij}) \widehat g^{(b)}_j + \sum_{j:i\to j} J^{(b)}_{ij}\,(\widehat g^{(b)}_j-g^{(b)}_j).
\end{equation}

Taking the squared norm gives
\begin{equation}
    \|\widehat g^{(b)}_i - g^{(b)}_i\|^2 = \left\| \sum_{j:i\to j} (\widehat J^{(b)}_{ij}-J^{(b)}_{ij}) \widehat g^{(b)}_j \right\|^2 + \left\| \sum_{j:i\to j} J^{(b)}_{ij} (\widehat g^{(b)}_j-g^{(b)}_j) \right\|^2 + 2\langle \cdot , \cdot \rangle .
\end{equation}
The cross-term $\langle \cdot , \cdot \rangle$ has zero expectation by conditional unbiasedness of $\widehat J^{(b)}_{ij}$ as stated in \cref{as:local_vjp_unbiased}.
Local unbiasedness of Jacobians also allows to write
\begin{equation}
    \|\widehat g^{(b)}_i - g^{(b)}_i\|^2 = \sum_{j:i\to j} \left\| (\widehat J^{(b)}_{ij}-J^{(b)}_{ij}) \widehat g^{(b)}_j \right\|^2 + \left\| \sum_{j:i\to j} J^{(b)}_{ij} (\widehat g^{(b)}_j-g^{(b)}_j) \right\|^2 .
\end{equation}
Finally averaging over the batch (and using sample wise independance) and taking expectations yields
\begin{equation}
    \mathbb{E}\|\widehat g_i-g_i\|^2 = \sum_{j:i\to j} \mathbb{E}\left\| \frac{1}{B} \sum_{b=1}^B (\widehat J^{(b)}_{ij}-J^{(b)}_{ij}) \widehat g^{(b)}_j \right\|^2 
    + \mathbb{E} \left\|\frac{1}{B} \sum_{b=1}^B \sum_{j:i\to j} J^{(b)}_{ij} (\widehat g^{(b)}_j-g^{(b)}_j) \right\|^2 .
\end{equation}

\end{proof}

\subsection{Proofs of \cref{sec:mask}}

\begin{lemma}[Restatement of \cref{thm:svd_sketch}]
    Let $M$ be a fixed matrix in $\dR^{m\times n}$. Let $q=\min(m,n)$ and denote the SVD of $M$ by $M=\sum_{i=1}^q \sigma_i u_i v_i^\top$, where $\sigma_i$, $u_i$, $v_i$ are respectively its singular values, left and right singular vectors. Then among matrices $S$ of rank at most $r$ where $r<q$, such that $\mathbb{E}[S]=M$, the choice that minimizes the error in squared Frobenius norm $\dE\|M-S\|^2_F$ is obtained by letting 
    $$
    S=\sum_{i=1}^q \sigma_i \frac{Z_i}{p_i}u_i v_i^\top$$
    where the $Z_i$ are correlated Bernoulli random variables with respective parameters $p_i$. The $Z_i$ are such that $\sum_{i=1}^q Z_i\equiv r$, and the $\{p_i\}_{i\in [q]}$ are minimizers of $\sum_{i=1}^q \sigma^2_i/p_i$ among weights $p_i\in [0,1]$ that sum to $r$.   
\end{lemma}

Let us first state the following Lemma:
\begin{lemma}
For fixed $i\in [q]$, the function $g_i$ defined on $\dR^{m\times n}$ by
$$
g_i(M):=\sum_{j=1}^i\sigma_j(M),
$$
that is the function returning the sum of the top-$i$ singular values  is convex.
\end{lemma}
This follows for instance from Theorem 4.3.27, p. 194 in \citet{horn13}.

\begin{proof}[Proof of \cref{thm:svd_sketch}]
For all $i\in [q]$, we have by Jensen's inequality:
\begin{equation}\label{eq:weak_majorization_1}
\dE g_i(S)=\dE \sum_{j=1}^i \sigma_j(S)\ge g_i(\dE(S))=g_i(M)=\sum_{j=1}^i \sigma_j.
\end{equation}

Now, let $i_0$ be an integer such that
\begin{equation}
\label{eq:i_0}
\sigma_{i_0}> \frac{\sum_{j=i_0+1}^q \sigma_j}{r-i_0}\ge \sigma_{i_0+1},
\end{equation}
where if $\sigma_1\le r^{-1}\sum_{j=1}^q \sigma_j$, we let $i_0=0$.

Let us denote, for $i\in [i_0]$, $\dE \sigma_i(S)$ by $s_i$. We thus have
$$
s_1\ge \sigma_1,\; s_1+s_2\ge \sigma_1+\sigma_2,\ldots, \sum_{j=1}^{i_0}s_j \ge \sum_{j=1}^{i_0} \sigma_j.
$$
This property is also known as the fact that vector $(\sigma_j)_{j\in [i_0]}$ is weakly majorized by vector $(s_j)_{j\in [i_0]}$, also denoted by $\sigma \prec_w s$, see Definition (11), p. 12 in \citet{marshall11}. 

Split then the singular value decomposition of $S$ into the first $i_0$ singular values, $\sigma_1(S),\ldots,\sigma_{i_0}(S)$ and into the last non-zero singular values, $\sigma_{i_0+1}(S),\ldots, \sigma_{r}(S)$, and let $X$ be a uniformly random selection among these last $r-i_0$ singular values of $S$. Use now Inequality \eqref{eq:weak_majorization_1} with $i=q$ to get
$$
s_1+\ldots+ s_{i_0}+(r-i_0)\dE(X)\ge \sum_{j=1}^q \sigma_j,
$$
so that
$$
\dE(X)\ge \frac{1}{r-i_0}\left[\sum_{j=1}^q \sigma_j-\sum_{i=1}^{i_0}s_i\right].
$$
This implies
\begin{eqnarray}\label{eq:step_1}
\dE\|S\|_F^2=\sum_{i=1}^r \dE[\sigma_i(S)^2]\ge \sum_{i=1}^{i_0} s_i^2+(r-i_0)(\dE(X))^2
\\
\ge \sum_{i=1}^{i_0} s_i^2+\frac{1}{r-i_0}\left[ \sum_{j=1}^q \sigma_j-\sum_{i=1}^{i_0}s_i\right]^2.
\end{eqnarray}
Introduce the function
$$
h((x_i)_{i\in [i_0]}):=\sum_{i=1}^{i_0} x_i^2+\frac{1}{r-i_0}\left[ \sum_{j=1}^q \sigma_j-\sum_{i=1}^{i_0}x_i\right]^2,
$$
and consider the convex set 
$$
A:=\{(x_i)_{i\in [i_0]}\in \dR^{i_0}: h((x_i)_{i\in [i_0]}) < h((\sigma_i)_{i\in [i_0]})\}.
$$
Let us show that $A$ does not intersect with the set 
$$
B:=\{(x_1\ge \cdots\ge x_{i_0}: \sigma \prec_w x\}.
$$
To that end, note that both sets are convex. Note that, at $x=\sigma$, the gradient of $h$ reads
$$
\nabla h(\sigma)=2\left(\sigma_i -\frac{1}{r-i_0}\sum_{j=i_0+1}^q \sigma_j\right)_{i\in [i_0]}.
$$
Clearly, for all $x\in A$, 
$$
\langle \nabla h(\sigma), x-\sigma\rangle <0.
$$
Let now $x\in B$. To show that $A$ and $B$ are disjoint, it will suffice to prove that
$$
\forall x\in B,\; \langle \nabla h(\sigma),x-\sigma\rangle \ge 0.
$$
Let 
$$
\alpha:=\frac{1}{r-i_0}\sum_{j=i_0+1}^q \sigma_j.
$$
Write then
$$
\begin{array}{l}
\langle \nabla h(\sigma),x-\sigma\rangle=\sum_{i=1}^{i_0} (x_i -\sigma_i)\left[\sigma_i -\alpha\right]\\
=
\sum_{i=1}^{i_0} (x_i -\sigma_i)\left[\sigma_{i_0}-\alpha+ \sum_{j=i}^{i_0 -1}(\sigma_j -\sigma_{j+1})\right]\\
=(\sigma_{i_0}-\alpha)\sum_{i=1}^{i_0}(x_i-\sigma_i)\\
\;+\sum_{j=1}^{i_0-1}(\sigma_j-\sigma_{j+1})\sum_{i=1}^{j}(x_i-\sigma_i).
\end{array}
$$
The property $\sigma\prec_w x$ implies $\sum_{i=1}^{j}(x_i-\sigma_i)\ge 0$ for all $j\in [i_0]$. The terms $(\sigma_j-\sigma_{j+1})$ are non-negative by assumption. Finally, the assumption \eqref{eq:i_0} ensures that $\sigma_{i_0}-\alpha>0$. Thus all terms in the last display are non-negative, so that for all $x\in B$, 
$$
\langle \nabla h(\sigma),x-\sigma\rangle\ge 0.
$$
We thus have, since  $\sigma\prec_w s$, that $h(s)\ge h(\sigma)$. However, \eqref{eq:step_1} ensures that $\dE\|S\|_F^2\ge h(s)$. 

By expliciting $h(\sigma)$, we get the lower bound:
\begin{equation}
\label{eq:lower_bound}
    \dE \|S\|_F^2 \ge \sum_{i=1}^{i_0}\sigma_i^2 + \frac{1}{r-i_0}\left[\sum_{j=i_0+1}^q \sigma_j\right]^2.
\end{equation}

We now describe the construction of the random matrix $S$ that achieves equality in this lower bound. Recall that the SVD of $M$ be given by
$$
M=\sum_{i=1}^q \sigma_i u_i v_i^\top.
$$
Define for $i=i_0+1,\ldots, q$ the quantity
$$
p_i:= (r-i_0) \frac{\sigma_i}{\sum_{j=i_0+1}^q \sigma_j}\cdot
$$
By the second inequality in the definition \eqref{eq:i_0} of $i_0$, it follows that 
$$
p_i \in [0,1], \; i\in \{i_0+1,\ldots,q\}.
$$
Let now $U_1$ be uniformly distributed over $[0,1]$, and let 
$$
U_j=U_1+ j-1, j=2,\ldots, r-i_0.
$$

Finally, we define
$$
\forall i\in \{i_0+1,\ldots, q\}, \;  Z_i=\left\{\begin{array}{ll}1&\hbox{ if }\displaystyle \exists j\in [r-i_0]: U_j\in \left[\sum_{l=i_0+1}^{i-1}p_l, \sum_{l=i_0+1}^i p_l\right],\\
0&\hbox{otherwise}.
\end{array}
\right.
$$
By construction, we have $\sum_{i=i_0+1}^q p_i=r-i_0$, so that each $Z_i$ is Bernoulli with parameter $p_i$, and moreover
$$
\sum_{i=i_0+1}^q Z_i =r-i_0\hbox{ almost surely}.
$$
We then let
$$
S:=\sum_{i=1}^{i_0}\sigma_i u_i v_i^\top + \sum_{i=i_0+1}^q Z_i\frac{\sigma_i}{p_i}u_i v_i^\top.
$$
The fact that $\dE S=M$ follows directly from the fact that the $Z_i$ are Bernoulli with parameter $p_i$. Finally, we have
$$
\begin{array}{ll}
\dE\|S\|_F^2&\displaystyle=\sum_{i=1}^{i_0} \sigma_i^2+\sum_{i=i_0+1}^q\frac{\sigma_i^2}{p_i}\\
&\displaystyle =\sum_{i=1}^{i_0} \sigma_i^2+\sum_{i=i_0+1}^q \sigma_i^2\frac{\sum_{j=i_0+1}^q \sigma_j}{(r-i_0) \sigma_i}\\
&\displaystyle =\sum_{i=1}^{i_0} \sigma_i^2 +\frac{1}{r-i_0}\left[\sum_{j=i_0+1}^q \sigma_j\right]^2
\end{array}
$$

We have thus exhibited a random matrix $S$ of rank at most $r$, satisfying $\mathbb{E}[S]=M$, that achieves the lower bound \eqref{eq:lower_bound}.
Consequently, this bound is tight and characterizes the minimum possible value of $\mathbb{E}\|M-S\|_F^2$ among all unbiased rank-$r$ sketches.

Moreover, for any sketch of the form
\[
S = \sum_{i=1}^q \sigma_i \frac{Z_i}{p_i} u_i v_i^\top,
\qquad
\mathbb{P}(Z_i=1)=p_i,\quad \sum_{i=1}^q Z_i = r \,
\]
we have
\[
\mathbb{E}\|S\|_F^2 = \sum_{i=1}^q \frac{\sigma_i^2}{p_i}.
\]
Therefore, minimizing $\mathbb{E}\|M-S\|_F^2$ over unbiased rank-$r$ sketches is equivalent to solving the convex optimization problem
\[
\min \sum_{i=1}^q \frac{\sigma_i^2}{p_i}
\quad \text{subject to} \quad
\sum_{i=1}^q p_i = r,\quad 0 < p_i \le 1.
\]

The weights $(p_i)$ defined above attain this minimum. Since the objective is convex and the constraints are affine, optimality follows from the Karush–Kuhn–Tucker conditions, which yield the thresholding structure
\[
p_i^\star = \min\!\left(1, \sqrt{\frac{\sigma_i^2}{\lambda}}\right),
\]
for some $\lambda>0$ chosen so that $\sum_i p_i^\star = r$. This exactly corresponds to the definition of $i_0$ and the probabilities constructed above, completing the proof.

\end{proof}

\begin{proposition}[Restatement of \cref{prop:optimal-column-sketch}]
Let $J\in\mathbb{R}^{m\times n}$ and $\Gamma_{\mathcal{B}}$ defined as above (\ie the empirical second moment matrix of the batch $\mathcal{B}$). 
Define the symmetric matrix
\begin{equation}
    \Gamma_{\mathcal{B}}^{1/2} J^\top J \Gamma_{\mathcal{B}}^{1/2} \;=\; U\,\Sigma\,U^\top,
\end{equation}
where $U$ is orthogonal and $\Sigma=\mathrm{diag}(\sigma^2_1,\ldots,\sigma^2_n)$ with $\sigma_1\ge \cdots \ge \sigma_n\ge 0$.
Among all random matrices $R$ with $\mathbb{E}[R]=I_n$ and
whose 
rank is bounded by $r$,
a minimizer of the $\mathsf{L}^2$ cost \Cref{eq:l2_cost} is attained by sketches
that are diagonal in the eigenbasis of $\Gamma_{\mathcal{B}}^{1/2}J^\top J\Gamma_{\mathcal{B}}^{1/2}$:
\begin{equation}
    \begin{aligned}
        R^* &= \Gamma_{\mathcal{B}}^{1/2}\,U\,B\,U^\top\,\Gamma_{\mathcal{B}}^{-1/2} \\
        B &= \mathrm{diag}(\frac{z_1}{p^*_1},\dots,\frac{z_n}{p^*_n}), \text{ with } z_i \sim \mathcal{B}(p^*_i)
    \end{aligned}
\end{equation}
with $\mathcal{B}(p)$ being the Bernoulli distribution of parameter $p$, and with selection probabilities $\{p^*_i\}_{i=1}^n$ minimizing $\sum_{i=1}^n \sigma^2_i/p_i$ under the constraint $\sum_i p_i=r$.
\end{proposition}

\begin{proof}
    We want to minimize the distortion
    \begin{equation}
        \mathcal{L}(R) = \frac{1}{B} \, \sum_{b=1}^B \,  \mathbb{E}[\|J_{ij}^{(b)} \widehat g_j^{(b)} - J_{ij}^{(b)} R \, \widehat g_j^{(b)} \|^2],
    \end{equation}
    as defined in \cref{eq:l2_cost}.
    As we stated in the main text, via the sketch of proof, the use of \cref{lem:linear_distortion} for linear settings and the constraint $\mathbb{E}[R] = I_n$ yield
    \begin{equation}
        \mathcal{L}(R) = \mathbb{E}\left[ \Tr \big(J R\Gamma_{\mathcal B}R^\top J^\top\big)\right].
    \end{equation}
    Leveraging the fact that, as a second moment matrix $\Gamma_\mathcal{B} := \frac{1}{B}GG^\top$ is positive definite, so $\Gamma_\mathcal{B}^{1/2}$ exists and we write $S=JR \Gamma_\mathcal{B}^{1/2}$.
    We then have
    \begin{equation}
        \mathcal{L}(R) = \mathbb{E}[ \Tr (SS^T)] = \mathbb{E}[\|S\|^2].
    \end{equation}
    Since we look at sketches $R$ such as $\mathrm{rank}(R) \le r$, then $S$ is of rank at most $r$ as well.
    Also, since $\mathbb{E}[R] = I_n$, we define $M := \mathbb{E}[S] = J \Gamma^{1/2}$, and the minimization of $\mathcal{L}$ is exactly the problem stated in \cref{prop:optimal-column-sketch}.
    Applying \cref{thm:svd_sketch} to $M$ we get that an optimal rank-$r$ unbiased sketch $S$ is obtained by sampling in its singular vector basis.
    Let $M = P\,\Sigma^{1/2}\,U^\top$ be an SVD of $M$, equivalently
    \[
        M^\top M = U\,\Sigma\,U^\top
        \qquad\text{with}\qquad
        M^\top M = \Gamma_{\mathcal B}^{1/2} J^\top J \Gamma_{\mathcal B}^{1/2},
    \]
    which matches \eqref{eq:basis_svd}. Therefore, an optimal choice is
    \begin{equation}
        S^\star \;=\; P\,\Sigma^{1/2}\,U^\top\,B
        \qquad\text{with}\qquad
        B=\mathrm{diag}\!\left(\frac{z_1}{p_1^\star},\ldots,\frac{z_n}{p_n^\star}\right),
        \ \ z_i\sim\mathcal{B}(p_i^\star),
        \ \ \sum_{i=1}^n z_i = r \,
        \label{eq:Sstar-form}
    \end{equation}
    where the probabilities $p_i^\star$ solve the convex program stated in \cref{prop:optimal-column-sketch}.
    
    Finally, since $S^\star = J R^\star \Gamma_{\mathcal B}^{1/2}$ and
    $M = J\Gamma_{\mathcal B}^{1/2} = P\Sigma^{1/2}U^\top$, we have
    \begin{equation}
    S^\star = M B = J\Gamma_{\mathcal B}^{1/2} B.
    \end{equation}
    Identifying both expressions yields a choice for $R^\star$
    \begin{equation}
    R^\star \Gamma_{\mathcal B}^{1/2} = \Gamma_{\mathcal B}^{1/2} B,
    \end{equation}
    hence
    \begin{equation}
    R^\star = \Gamma_{\mathcal B}^{1/2} B \Gamma_{\mathcal B}^{-1/2}.
    \end{equation}
    Writing $B$ in the eigenbasis of $\Gamma_{\mathcal B}^{1/2}J^\top J\Gamma_{\mathcal B}^{1/2}$ finally gives
    \begin{equation}
    R^\star
    =
    \Gamma_{\mathcal B}^{1/2}
    U\,\mathrm{diag}\!\left(\frac{z_1}{p_1^\star},\ldots,\frac{z_n}{p_n^\star}\right)
    U^\top
    \Gamma_{\mathcal B}^{-1/2},
    \end{equation}
    which is exactly the form \eqref{eq:opt-R-form}.
    
\end{proof}

\begin{lemma}[Restatement of \cref{cor:diag-mask}]
    We now restrict $R$ to be diagonal, $R=\mathrm{diag}(r_1,\ldots,r_n)$ with
    \[
    r_i \;=\; z_i/p_i, \quad
    z_i \sim \mathcal{B}(p_i)\, \quad \text{independent,}\quad
    p_i \in ]0,1],
    \]
    so that $\mathbb{E}[R]=I_n$.  For such diagonal mask, under the expected rank constraint $\sum_{i=1}^n p_i\le r$, 
    the minimal $\mathsf{L}^2$ distortion \Cref{eq:l2_cost} is obtained be choosing probabilities $p_i$ that minimize 
    $
    \inf_{p}\sum_{i=1}^n \frac{a_i}{p_i} $ under the constraint $\sum_i p_i\le r$, where 
    $
    a_i \;:=\; (\Gamma_{\mathcal{B}})_{ii}\,(J^\top J)_{ii}$.
\end{lemma}

\begin{proof}
With $R=\mathrm{diag}(r_1,\dots,r_n)$ (hence $(I - R)^\top=(I - R)$), expand
\begin{equation}
   \begin{aligned}
       \Tr\!\big(J^\top J\,(I - R) \Gamma_{\mathcal{B}} (I - R) \big) &= \sum_{i,j=1}^n (J^\top J)_{ij}\, (1 - r_i) (1 - r_j) (\Gamma_{\mathcal{B}})_{ij} \\
       &= \sum_{i=1}^n (J^\top J)_{ii}\, (1-r_i)^2  (\Gamma_{\mathcal{B}})_{ii} + \sum_{i\neq j} (J^\top J)_{ij}\, (1-r_i) (1-r_j) (\Gamma_{\mathcal{B}})_{ij}.
   \end{aligned}
\end{equation}

Under the diagonal mask model $r_i=z_i/p_i$ with $z_i\sim \mathcal{B} (p_i)$ independent of $G$, we have:
\begin{equation}
   \begin{aligned}
       \mathbb{E}[1 - r_i] &=0, \\
       \mathbb{E}[(1 - r_i) (1 - r_j)] &=\mathbb{E}[(1 - r_i)]\mathbb{E}[(1 - r_j)]=0\ (i\neq j), \\
       \mathbb{E}[(1 - r_i)^2] &=\frac{1}{p_i} - 1\quad(p_i>0).
   \end{aligned}    
\end{equation}

Taking expectations yields
\begin{equation}
   \mathbb{E}\!\left[\Tr\!\big(J^\top J\,(I-R) \Gamma_{\mathcal{B}} (I-R)\big)\right] = \sum_{i=1}^n (J^\top J)_{ii}\, \big(\frac{(\Gamma_{\mathcal{B}})_{ii}}{p_i} - (\Gamma_{\mathcal{B}})_{ii} \big).
\end{equation}

Since we aim at minimizing with $p_i$ as the variable, we get rid of the second term of the difference, hence the minimization objective

\begin{equation}
   \sum_{i=1}^n \frac{a_i}{p_i}, \quad \text{with } a_i \;:=\; (\Gamma_{\mathcal{B}})_{ii}\,(J^\top J)_{ii}.
\end{equation}

Therefore, minimizing $\mathcal{L}$ over diagonal masks with budget
$\sum_{i=1}^n p_i\le r$ is equivalent to
\begin{equation}
   \min_{\,p\in[0,1]^n,\ \sum_i p_i\le r}\ \sum_{i=1}^n \frac{a_i}{p_i}.
\end{equation}

This matches the constraints stated in Proposition~\ref{prop:optimal-column-sketch} and concludes the proof.
\end{proof}

\newpage

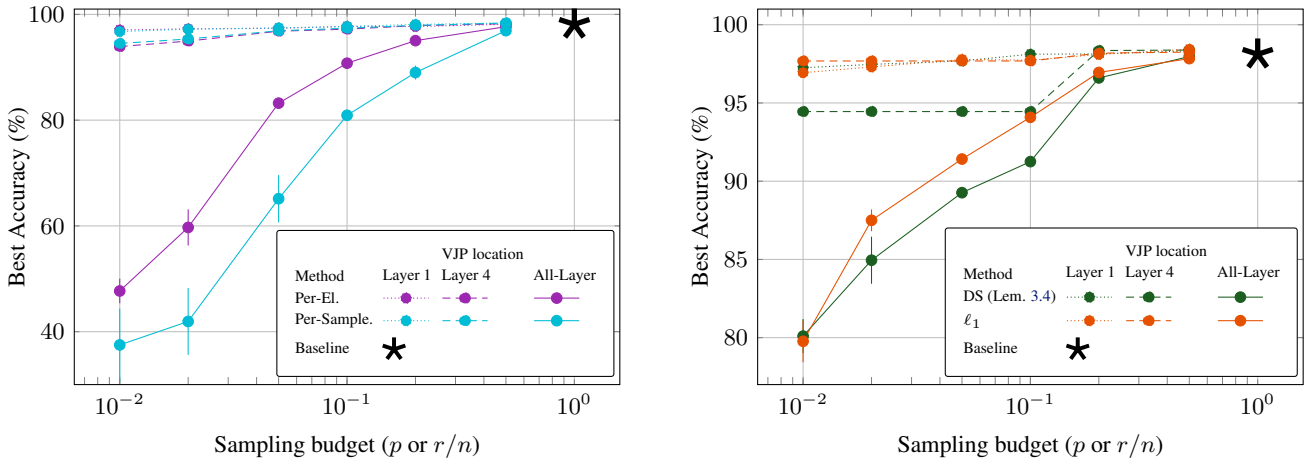
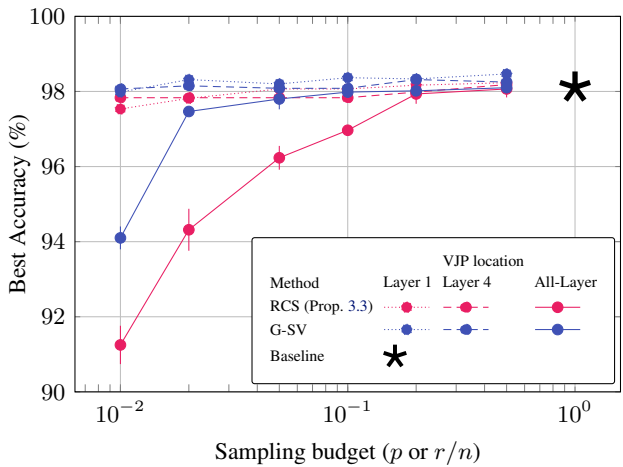
\begin{figure*}[t]
    \centering

    \begin{subfigure}[t]{0.47\textwidth}
        \centering
        \input{figures_tex/variance_prop_masks}
        \caption{Masking Methods}
        \label{fig:variance_masking}
    \end{subfigure}
    \hfill
    \begin{subfigure}[t]{0.47\textwidth}
        \centering
        \input{figures_tex/variance_prop_coord}
        \caption{Coordinate-based Sketching Methods}
        \label{fig:variance_coordinate}
    \end{subfigure}

    \vspace{1em}

    \begin{subfigure}[t]{0.47\textwidth}
        \centering
        \input{figures_tex/variance_prop_spectral}
        \caption{Spectral-based Sketching Methods}
        \label{fig:variance_spectral}
    \end{subfigure}

    \caption{Impact of VJP Approximation Location in MLPs.}
    \label{fig:variance_location}
\end{figure*}
\section{Additional Numerical Details}

\subsection{Additional Variance Location Result on MLP}

In the main body of the paper, we exclusively considered experiments in which our methods are implemented for all layers of the MLP.
In this appendix, we complement these results by comparing full-layer VJP approximation with variants, where it is performed either only to the first layer or only to the last layer.
Results are reported in \cref{fig:variance_location}.
As one might intuitively expect, approximating local VJPs solely in the last layer leads to a more pronounced degradation in accuracy than when noise is injectes only in the first layer.
Beyond this empirical observation, these results may provide practical insights for distributed training settings and especially, in straggler-mitigation scenarios where one may wish to apply VJP approximations selectively at a slower compute node rather than uniformly across the network.

\subsection{Details on ViT and BagNet}

We provide here more details about the training of the larger architectures we used in \cref{sec:numerical_analysis} to perform experiments on the CIFAR-10 \citep{cifar10} dataset.
For both models, the data was processed through the same simple augmentation process consisting in random $32 \times 32$ cropping and horizontal flipping.
Both training used a batch-size of 128.

\paragraph{BagNet-17.}
The first architecture we looked at was BagNet-17 \citep{brendel2019approximating}.
To establish a baseline, we used an SGD optimizer, with a momentum of 0.9, a learning rate of $10^{-1.5} \approx 0.032$, and a weight decay of $10^{-3}$.
We also used a cosine scheduler, that reduces the learning rate to $10^{-5}$ over the course of the training.
As mentioned in the main article, VJP approximation methods were employed on every linear (or equivalent, like $1 \times 1$ convolutions) layer, except the final classification layer and the initial input projection.

\paragraph{Visual Transormer (ViT).}
We also trained some ViT architecture \cite{dosovitskiy2021imageworth16x16words}.
Regarding the architecture details, we used an embedding dimension of 192, a MLP size of 1024, a network depth of 9, and 12 attention heads.
Images were processed with a patch size of 4.
Finally, to establish a baseline, we trained our model using AdamW \citep{loshchilov2019decoupledweightdecayregularization} optimizer, with a learning rate of $3 \times 10^{-4}$ and a weight decay of $0.05$.
We also used a dropout parameter of 0.1, and a cosine scheduler with a $10$ epochs warmup.

\clearpage
\newpage

\section{Pseudo-code}

\subsection{Practical considerations}

We provide here the pseudocode for the different masking procedures implemented and discussed in the main paper.
Note that, in most practical deep learning frameworks, operations are performed on tensors that represent batches of data.
Moreover, in these frameworks, row vectors are typically used, in contrast to the usual theoretical setup that considers column vectors.

For example, in the usual mathematical setup, a linear layer is written as:
\begin{equation}
    y = Wx + b,
\end{equation}
with $x \in \mathbb{R}^{d_{\mathrm{in}}}$, $W \in \mathbb{R}^{d_{\mathrm{out}} \times d_{\mathrm{in}}}$, $b \in \mathbb{R}^{d_{\mathrm{out}}}$, and $y \in \mathbb{R}^{d_{\mathrm{out}}}$, where $d_{\mathrm{in}}$ and $d_{\mathrm{out}}$ denote the input and output dimensions of the layer, respectively.

In the practical setup, however, the linear layer is implemented as:
\begin{equation}
    y = x W^\top + b,
\end{equation}
with $x \in \mathbb{R}^{B \times d_{\mathrm{in}}}$, $W \in \mathbb{R}^{d_{\mathrm{out}} \times d_{\mathrm{in}}}$, $b \in \mathbb{R}^{d_{\mathrm{out}}}$ (broadcasted across the batch dimension), and $y \in \mathbb{R}^{B \times d_{\mathrm{out}}}$, where $B$ denotes the batch size.

\subsection{Utils}

This subsection shows code to find optimal probabilities $p_i^\star$ for data-dependent sketches (\cref{subsec:sketching_operators}), by solving the convex program stated in \eqref{eq:convex_program} as well as the algorithm that performs \emph{exact}-r Bernoulli correlated sampling, as described in \cref{prop:optimal-column-sketch}.

\begin{algorithm}[h]
\caption{Convex program solving for optimal probabilities \eqref{eq:convex_program}}
\label{alg:convex_program}
\begin{algorithmic}[1]
\REQUIRE Importance weights $w_i > 0$
\ENSURE Optimal sampling probabilities $p_i^\star \in (0, 1]$ with $\sum_{i=1}^n p_i^\star \le r$.

\STATE $t_i \leftarrow \sqrt{w_i}$ for all $i$
\STATE Sort $t$ in decreasing order: $t_{(1)} \ge \cdots \ge t_{(n)}$
\STATE Compute suffix sums $S_k \leftarrow \sum_{i=k}^n t_{(i)}$

\FOR{$k = 0,\dots,n-1$}
    \STATE $\text{remainder} \leftarrow r - k$
    \IF{$\text{remainder} \le 0$}
        \STATE \textbf{break}
    \ENDIF
    \STATE $\sqrt{\lambda} \leftarrow S_{k+1} / \text{remainder}$
    \IF{$k = 0$ \OR $t_{(k)} \ge \sqrt{\lambda}$}
        \IF{$t_{(k+1)} \le \sqrt{\lambda}$}
            \STATE $k^\star \leftarrow k$
            \STATE \textbf{break}
        \ENDIF
    \ENDIF
\ENDFOR

\FOR{$i = 1,\dots,n$}
    \STATE $p^\star_{(i)} \leftarrow \min\!\left(1,\; t_{(i)} / \sqrt{\lambda}\right)$
\ENDFOR

\STATE Undo the sorting permutation
\STATE \textbf{return} $p^\star$

\end{algorithmic}
\end{algorithm}

\begin{algorithm}[H]
\caption{Correlated exact-$r$ sampling}
\label{alg:correlated_sampling}
\begin{algorithmic}[1]
\REQUIRE Probabilities $p \in (0,1]^n$ with $\sum_{i=1}^n p_i = r$
\ENSURE Exactly $r$ distinct sampled indices

\STATE Compute cumulative sums:
\STATE \hspace{1em} $P_j \leftarrow \sum_{i=1}^j p_i$ for $j=1,\dots,n$
\STATE $P_n \leftarrow r$ \COMMENT{numerical safety}

\STATE Sample $u \sim \mathrm{Uniform}(0,1]$

\FOR{$\ell = 0,\dots,r-1$}
    \STATE $t \leftarrow u + \ell$
    \STATE Find smallest $j$ such that $P_j \ge t$
    \STATE Select index $j$
\ENDFOR

\STATE \textbf{return} the $r$ selected indices
\end{algorithmic}
\end{algorithm}

\subsection{Masking Methods}

This subsection contains pseudocode for the three masking methods presented in \cref{subsec:masking_operators}.

\begin{algorithm}[h!]
\caption{Per-Element Masking Backward Pass in a Linear Layer}
\label{alg:per_el_masking}
\begin{algorithmic}[1]
\REQUIRE Output gradient $\mathbf{G} \triangleq \frac{\partial \mathcal{L}}{\partial \mathbf{Y}} \in \mathbb{R}^{B \times d_{\text{out}}}$
\REQUIRE Cached tensors $\mathbf{X} \in \mathbb{R}^{B \times d_{\text{in}}}$, $\mathbf{W} \in \mathbb{R}^{d_{\text{out}} \times d_{\text{in}}}$, $\mathbf{b} \in \mathbb{R}^{d_{\text{out}}}$
\REQUIRE Sampling probability $p \in (0,1]$
\ENSURE Gradients $\frac{\partial \mathcal{L}}{\partial \mathbf{X}}$, $\frac{\partial \mathcal{L}}{\partial \mathbf{W}}$, $\frac{\partial \mathcal{L}}{\partial \mathbf{b}}$

\vspace{0.25em}
\STATE \textbf{// Generate element-wise masks}
\STATE $\mathbf{M}_W \sim \mathrm{Bernoulli}(p) \in \{0,1\}^{d_{\text{out}} \times d_{\text{in}}}$ \hfill (weight mask)
\STATE $\mathbf{M}_X \sim \mathrm{Bernoulli}(p) \in \{0,1\}^{B \times d_{\text{in}}}$ \hfill (input mask)

\vspace{0.25em}
\STATE \textbf{// Gradient with respect to the input}
\STATE $\hat{\mathbf{W}} \gets \mathbf{W} \odot \mathbf{M}_W$
\STATE $\frac{\partial \mathcal{L}}{\partial \mathbf{X}} \gets \frac{1}{p}\, \mathbf{G}\hat{\mathbf{W}}$

\vspace{0.25em}
\STATE \textbf{// Gradient with respect to the weights}
\STATE $\hat{\mathbf{X}} \gets \mathbf{X} \odot \mathbf{M}_X$
\STATE $\frac{\partial \mathcal{L}}{\partial \mathbf{W}} \gets \frac{1}{p}\, \mathbf{G}^\top \hat{\mathbf{X}}$

\vspace{0.25em}
\STATE \textbf{// Gradient with respect to the bias}
\STATE $\frac{\partial \mathcal{L}}{\partial \mathbf{b}} \gets \sum_{i=1}^{B} \mathbf{G}[i,:]$

\STATE \textbf{return} $\left(\frac{\partial \mathcal{L}}{\partial \mathbf{X}},\frac{\partial \mathcal{L}}{\partial \mathbf{W}},\frac{\partial \mathcal{L}}{\partial \mathbf{b}}\right)$
\end{algorithmic}
\end{algorithm}

\begin{algorithm}[h!]
\caption{Per-Sample Masking Backward Pass in a Linear Layer}
\label{alg:per_sample_masking}
\begin{algorithmic}[1]
\REQUIRE Output gradient $\mathbf{G} \triangleq \frac{\partial \mathcal{L}}{\partial \mathbf{Y}} \in \mathbb{R}^{B \times d_{\text{out}}}$
\REQUIRE Cached tensors $\mathbf{X} \in \mathbb{R}^{B \times d_{\text{in}}}$, $\mathbf{W} \in \mathbb{R}^{d_{\text{out}} \times d_{\text{in}}}$, $\mathbf{b} \in \mathbb{R}^{d_{\text{out}}}$
\REQUIRE Sampling probability $p \in (0,1]$
\ENSURE Gradients $\frac{\partial \mathcal{L}}{\partial \mathbf{X}}$, $\frac{\partial \mathcal{L}}{\partial \mathbf{W}}$, $\frac{\partial \mathcal{L}}{\partial \mathbf{b}}$

\vspace{0.25em}
\STATE \textbf{// Generate row-wise mask (sample-level masking)}
\STATE $\mathbf{m} \sim \mathrm{Bernoulli}(p) \in \{0,1\}^{B}$
\STATE $\mathbf{M}_G \gets \mathbf{m}\,\mathbf{1}_{d_{\text{out}}}^\top \in \{0,1\}^{B \times d_{\text{out}}}$

\vspace{0.25em}
\STATE \textbf{// Apply mask and rescale (unbiased estimator)}
\STATE $\hat{\mathbf{G}} \gets \frac{1}{p}\, (\mathbf{G} \odot \mathbf{M}_G)$

\vspace{0.25em}
\STATE \textbf{// Backward pass with masked gradient}
\STATE $\frac{\partial \mathcal{L}}{\partial \mathbf{X}} \gets \hat{\mathbf{G}}\,\mathbf{W}$
\STATE $\frac{\partial \mathcal{L}}{\partial \mathbf{W}} \gets \hat{\mathbf{G}}^\top \mathbf{X}$
\STATE $\frac{\partial \mathcal{L}}{\partial \mathbf{b}} \gets \sum_{i=1}^{B} \hat{\mathbf{G}}[i,:]$

\STATE \textbf{return} $\left(\frac{\partial \mathcal{L}}{\partial \mathbf{X}},\frac{\partial \mathcal{L}}{\partial \mathbf{W}},\frac{\partial \mathcal{L}}{\partial \mathbf{b}}\right)$
\end{algorithmic}
\end{algorithm}

\begin{algorithm}[h!]
\caption{Per-Column Masking Backward Pass in a Linear Layer}
\label{alg:column_wise_backward}
\begin{algorithmic}[1]
\REQUIRE Output gradient $\mathbf{G} \triangleq \frac{\partial \mathcal{L}}{\partial \mathbf{Y}} \in \mathbb{R}^{B \times d_{\text{out}}}$
\REQUIRE Cached tensors $\mathbf{X} \in \mathbb{R}^{B \times d_{\text{in}}}$, $\mathbf{W} \in \mathbb{R}^{d_{\text{out}} \times d_{\text{in}}}$, $\mathbf{b} \in \mathbb{R}^{d_{\text{out}}}$
\REQUIRE Sparsity probability $p \in (0,1]$
\ENSURE Gradients $\frac{\partial \mathcal{L}}{\partial \mathbf{X}}$, $\frac{\partial \mathcal{L}}{\partial \mathbf{W}}$, $\frac{\partial \mathcal{L}}{\partial \mathbf{b}}$

\vspace{0.25em}
\STATE \textbf{// Generate column-wise mask}
\STATE $\mathbf{m} \sim \mathrm{Bernoulli}(p) \in \{0,1\}^{d_{\text{out}}}$
\STATE $\mathbf{M}_G \gets \mathbf{1}_B \mathbf{m}^\top \in \{0,1\}^{B \times d_{\text{out}}}$

\vspace{0.25em}
\STATE \textbf{// Apply mask and rescale (unbiased estimator)}
\STATE $\hat{\mathbf{G}} \gets \frac{1}{p}\, (\mathbf{G} \odot \mathbf{M}_G)$

\vspace{0.25em}
\STATE \textbf{// Backward pass with masked gradient}
\STATE $\frac{\partial \mathcal{L}}{\partial \mathbf{X}} \gets \hat{\mathbf{G}}\,\mathbf{W}$
\STATE $\frac{\partial \mathcal{L}}{\partial \mathbf{W}} \gets \hat{\mathbf{G}}^\top \mathbf{X}$
\STATE $\frac{\partial \mathcal{L}}{\partial \mathbf{b}} \gets \sum_{i=1}^{B} \hat{\mathbf{G}}[i,:]$

\STATE \textbf{return} $\left(\frac{\partial \mathcal{L}}{\partial \mathbf{X}},\frac{\partial \mathcal{L}}{\partial \mathbf{W}},\frac{\partial \mathcal{L}}{\partial \mathbf{b}}\right)$
\end{algorithmic}
\end{algorithm}

Technical note regardining \cref{alg:per_sample_masking}: The mask generated is column-wise on the gradients, in a \emph{practical} implementation, where notations are transposed compared to usual theoretical frameworks.
Hence the mask zeroes-out \emph{rows} of the gradients in mathematical notation...
Which is equivalent to zeroing columns of the Jacobian.

\subsection{Sketching Methods}

We now give the general sketching algorithm. Note that weight can be computed in various different ways.

This algorithm (\cref{alg:coordinate_sketch}) can be used for all sketching methods that operate in the canonical basis (\eg $\ell_1$, $\ell_2$, etc.), score computation just needs to be adapted.

\begin{algorithm}[h!]
\caption{Coordinate Sketch Backward Pass for a Linear Layer (Example with L1 Proxy)}
\label{alg:coordinate_sketch}
\begin{algorithmic}[1]
\REQUIRE Gradient \wrt output (matrix) $\mathbf{G} \triangleq \frac{\partial \mathcal{L}}{\partial \mathbf{Y}} \in \mathbb{R}^{B \times d_{\text{out}}}$
\REQUIRE Cached tensors $\mathbf{X} \in \mathbb{R}^{B \times d_{\text{in}}}$, $\mathbf{W} \in \mathbb{R}^{d_{\text{out}} \times d_{\text{in}}}$, (optional) $\mathbf{b} \in \mathbb{R}^{d_{\text{out}}}$
\REQUIRE Target number of selected columns $r$
\REQUIRE Subroutines: $\textsc{PStarFromScores}(\cdot)$ (Alg.~\ref{alg:convex_program}) and $\textsc{CorrelatedExactRSample}(\cdot)$ (Alg.~\ref{alg:correlated_sampling})
\ENSURE Gradients $\frac{\partial \mathcal{L}}{\partial \mathbf{X}}$, $\frac{\partial \mathcal{L}}{\partial \mathbf{W}}$, $\frac{\partial \mathcal{L}}{\partial \mathbf{b}}$

\vspace{0.25em}
\STATE \textbf{// Compute L1 based importance weights per output column}
\FOR{$j = 1,\dots,d_{\text{out}}$}
    \STATE $s_j \gets \left\|\mathbf{G}[:,j]\right\|_1^2$
\ENDFOR
\STATE $\mathbf{s} \gets (s_1,\dots,s_{d_{\text{out}}})$

\vspace{0.25em}
\STATE \textbf{// Compute optimal marginal probabilities and sample exactly $r$ columns}
\STATE $\mathbf{p}^\star \gets \textsc{PStarFromWeights}(\mathbf{s}, r)$
\STATE $\mathcal{I} \gets \textsc{CorrelatedExactRSample}(\mathbf{p}^\star, r)$ \hfill ($|\mathcal{I}|=r$)

\vspace{0.25em}
\STATE \textbf{// Form unbiased masked estimator of the output gradient}
\STATE $\mathbf{m} \gets \mathbf{0} \in \{0,1\}^{d_{\text{out}}}$
\FOR{each $j \in \mathcal{I}$}
    \STATE $m_j \gets 1$
\ENDFOR
\STATE $\hat{\mathbf{G}} \gets \mathbf{G} \odot (\mathbf{1}_B \mathbf{m}^\top)$
\FOR{each $j \in \mathcal{I}$}
    \STATE $\hat{\mathbf{G}}[:,j] \gets \hat{\mathbf{G}}[:,j] / p^\star_j$ \hfill // Scaling to ensure unbiasedness
\ENDFOR

\vspace{0.25em}
\STATE \textbf{// Backward pass using the sketched output gradient}
\STATE $\frac{\partial \mathcal{L}}{\partial \mathbf{X}} \gets \hat{\mathbf{G}}\,\mathbf{W}$
\STATE $\frac{\partial \mathcal{L}}{\partial \mathbf{W}} \gets \hat{\mathbf{G}}^\top \mathbf{X}$
\IF{$\mathbf{b}$ is present}
    \STATE $\frac{\partial \mathcal{L}}{\partial \mathbf{b}} \gets \sum_{i=1}^{B} \hat{\mathbf{G}}[i,:]$
\ENDIF

\STATE \textbf{return} $\left(\frac{\partial \mathcal{L}}{\partial \mathbf{X}},\frac{\partial \mathcal{L}}{\partial \mathbf{W}},\frac{\partial \mathcal{L}}{\partial \mathbf{b}}\right)$
\end{algorithmic}
\end{algorithm}

\newpage


\end{document}

%% file: definitions.tex
\usepackage{graphicx} 
\usepackage{amsmath}
\usepackage{amsthm}
\usepackage{amssymb}
\usepackage{xcolor}
\usepackage{booktabs}
\usepackage{tabularx}
\usepackage{multirow}
\usepackage{subcaption}
\usepackage{adjustbox}
\usepackage{arydshln}
\usepackage{pgfplots}
\pgfplotsset{compat=1.18}
\usepackage{tikz}
\usetikzlibrary{matrix}

\newcommand{\killian}[1]{\textcolor{orange}{Killian: #1}}

\newcommand{\dE}{\mathbb{E}}
\newcommand{\Tr}{\hbox{Tr}}
\newcommand{\dR}{\mathbb{R}}
\newcommand{\Var}{\mathrm{Var}}

\newcommand{\ie}{i.e.\ }
\newcommand{\eg}{e.g.\ }
\newcommand{\wrt}{w.r.t.\ }
\newcommand{\st}{\mbox{\ s.t.\ }}

\pgfplotsset{
  every error bar/.append style={solid},
  every error mark/.append style={solid},
}

\definecolor{color1}{RGB}{233,30,99}    
\definecolor{color2}{RGB}{63,81,181}    
\definecolor{color3}{RGB}{0,188,212}    
\definecolor{color4}{RGB}{255,193,7}    
\definecolor{color5}{RGB}{76,175,80}    
\definecolor{color6}{RGB}{156,39,176}   
\definecolor{color7}{RGB}{33,33,33}     

\definecolor{baseline}{RGB}{33,33,33}          

\definecolor{col_ew_sparse}{RGB}{156,39,176}        
\definecolor{col_partial}{RGB}{0,188,212}          
\definecolor{col_uniform}{RGB}{120,144,156}        

\definecolor{col_lone}{RGB}{230,81,0}              
\definecolor{col_ltwo}{RGB}{0,150,136}             
\definecolor{col_var}{RGB}{255,193,7}             
\definecolor{col_lemma}{RGB}{27,94,32}            

\definecolor{col_gsv}{RGB}{63,81,181}              
\definecolor{col_prop}{RGB}{233,30,99}              


%% file: figures_tex/corr_vs_no_corr_2.tex
\begin{tikzpicture}
  \begin{axis}[
    title={},
    title style={font=\small},
    width=0.9\columnwidth,
    height=0.52\columnwidth,
    xlabel={Sampling budget ($p$ or $r/n$)},
    ylabel={Best Accuracy (\%)},
    xlabel style={font=\footnotesize, inner sep=0pt},
    ylabel style={font=\footnotesize, inner sep=0pt},
    xlabel shift=4pt,
    ylabel shift=-4pt,
    grid=major,
    xmode=log,
    ymin=75, ymax=101,
    legend pos = south east,
    legend columns=1,
    legend style={font=\tiny},
    scale only axis,
    tick label style={font=\footnotesize},
    major tick length=2pt,
  ]
   
    \addplot[color=col_lone, mark=*, mark size=1.2pt, densely dashed]
      table [x=p, y=best_val_accuracy_median, col sep=comma]
      {data/mlp_all/topr_l1_squared/all_layers.csv};
      \label{corr_plot:ind_lone}

    \addplot+[
      color=col_lone,
      draw=none,
      mark=none,
      forget plot,
      error bars/.cd, y dir=both, y explicit,
      error mark=none,
    ]
      table [x=p, y=best_val_accuracy_median, y error=std_acc, col sep=comma]
      {data/mlp_all/topr_l1_squared/all_layers.csv};

    \addplot[color=col_lone, mark=*, mark size=1.2pt]
      table [x=p, y=best_val_accuracy_median, col sep=comma]
      {data/mlp_all/corr_topr_l1_squared/all_layers.csv};
      \label{corr_plot:cor_lone}

    \addplot+[
      color=col_lone,
      draw=none,
      mark=none,
      forget plot,
      error bars/.cd, y dir=both, y explicit,
      error mark=none,
    ]
      table [x=p, y=best_val_accuracy_median, y error=std_acc, col sep=comma]
      {data/mlp_all/corr_topr_l1_squared/all_layers.csv};

    \addplot[color=col_lemma, mark=*, mark size=1.2pt, densely dashed]
      table [x=p, y=best_val_accuracy_median, col sep=comma]
      {data/mlp_all/diag_optimal_sec_mom/all_layers.csv};
      \label{corr_plot:ind_diag}

    \addplot+[
      color=col_lemma,
      draw=none,
      mark=none,
      forget plot,
      error bars/.cd, y dir=both, y explicit,
      error mark=none,
    ]
      table [x=p, y=best_val_accuracy_median, y error=std_acc, col sep=comma]
      {data/mlp_all/diag_optimal_sec_mom/all_layers.csv};

    \addplot[color=col_lemma, mark=*, mark size=1.2pt]
      table [x=p, y=best_val_accuracy_median, col sep=comma]
      {data/mlp_all/corr_diag_optimal_sec_mom/all_layers.csv};
      \label{corr_plot:cor_diag}

    \addplot+[
      color=col_lemma,
      draw=none,
      mark=none,
      forget plot,
      error bars/.cd, y dir=both, y explicit,
      error mark=none,
    ]
      table [x=p, y=best_val_accuracy_median, y error=std_acc, col sep=comma]
      {data/mlp_all/corr_diag_optimal_sec_mom/all_layers.csv};

    \addplot[color=col_gsv, mark=*, mark size=1.2pt, densely dashed]
      table [x=p, y=best_val_accuracy_median, col sep=comma]
      {data/mlp_all/svd_topr_squared/all_layers.csv};
      \label{corr_plot:ind_gsv}

    \addplot+[
      color=col_gsv,
      draw=none,
      mark=none,
      forget plot,
      error bars/.cd, y dir=both, y explicit,
      error mark=none,
    ]
      table [x=p, y=best_val_accuracy_median, y error=std_acc, col sep=comma]
      {data/mlp_all/svd_topr_squared/all_layers.csv};

    \addplot[color=col_gsv, mark=*, mark size=1.2pt]
      table [x=p, y=best_val_accuracy_median, col sep=comma]
      {data/mlp_all/corr_svd_topr_squared/all_layers.csv};
      \label{corr_plot:cor_gsv}

    \addplot+[
      color=col_gsv,
      draw=none,
      mark=none,
      forget plot,
      error bars/.cd, y dir=both, y explicit,
      error mark=none,
    ]
      table [x=p, y=best_val_accuracy_median, y error=std_acc, col sep=comma]
      {data/mlp_all/corr_svd_topr_squared/all_layers.csv};

    \addplot[color=col_prop, mark=*, mark size=1.2pt, densely dashed]
      table [x=p, y=best_val_accuracy_median, col sep=comma]
      {data/mlp_all/optimal_sec_mom/all_layers.csv};
      \label{corr_plot:ind_opt}

    \addplot+[
      color=col_prop,
      draw=none,
      mark=none,
      forget plot,
      error bars/.cd, y dir=both, y explicit,
      error mark=none,
    ]
      table [x=p, y=best_val_accuracy_median, y error=std_acc, col sep=comma]
      {data/mlp_all/optimal_sec_mom/all_layers.csv};

    \addplot[color=col_prop, mark=*, mark size=1.2pt]
      table [x=p, y=best_val_accuracy_median, col sep=comma]
      {data/mlp_all/corr_optimal_sec_mom/all_layers.csv};
      \label{corr_plot:cor_opt}

    \addplot+[
      color=col_prop,
      draw=none,
      mark=none,
      forget plot,
      error bars/.cd, y dir=both, y explicit,
      error mark=none,
    ]
      table [x=p, y=best_val_accuracy_median, y error=std_acc, col sep=comma]
      {data/mlp_all/corr_optimal_sec_mom/all_layers.csv};

    \addplot[color=black, mark=star, only marks, mark size=4.5pt,
      mark options={line width=2pt}]
      coordinates {(1, 98.11)};

    \coordinate (legend) at (axis description cs:0.99,0.02);
\end{axis}

\matrix [
          draw,
          rounded corners=1pt,
          fill=white,
          fill opacity=0.55,
          text opacity=1,
          matrix of nodes,
          anchor=south east,
          nodes={font=\tiny},
          column sep=5pt,
          row sep=-4pt,
          cells ={anchor=west}
        ] at (legend) {
                                              & $Corr.$                & $Indep.$ \\
            $\ell_1$                          & \ref{corr_plot:cor_lone}   & \ref{corr_plot:ind_lone}   \\
            RCS (Prop. \ref{prop:optimal-column-sketch}) & \ref{corr_plot:cor_diag}   & \ref{corr_plot:ind_diag}   \\
            G-SV                              & \ref{corr_plot:cor_gsv}   & \ref{corr_plot:ind_gsv}  \\
            DS (Lem. \ref{cor:diag-mask}) & \ref{corr_plot:cor_opt}   & \ref{corr_plot:ind_opt}   \\
            Baseline & \tikz{\draw plot[only marks, mark=star, mark size=4.5pt, mark options={line width=1.5pt}] coordinates {(0,0)};} &  \\
        };

\end{tikzpicture}

%% file: figures_tex/uninformed_vs_score.tex
\begin{tikzpicture}
  \begin{axis}[
    title={},
    title style={font=\small},
    width=0.9\columnwidth,
    height=0.52\columnwidth,
    xlabel={Sampling budget ($p$ or $r/n$)},
    ylabel={Best Accuracy (\%)},
    xlabel style={font=\footnotesize, inner sep=0pt},
    ylabel style={font=\footnotesize, inner sep=0pt},
    xlabel shift=4pt,
    ylabel shift=-4pt,
    grid=major,
    xmode=log,
    ymin=28, ymax=103,
    legend pos = south east,
    legend columns=1,
    legend style={font=\tiny},
    scale only axis,
    tick label style={font=\footnotesize},
    major tick length=2pt,
  ]
    \addplot[color=col_ew_sparse, mark=*, mark size=1.2pt]
      table [x=p, y=best_val_accuracy_median, col sep=comma]
      {data/mlp_all/ew_sparse_jac/all_layers.csv};
      \label{unif_plot:ew}

    \addplot+[
      color=col_ew_sparse,
      draw=none,
      mark=none,
      forget plot,
      error bars/.cd, y dir=both, y explicit,
      error mark=none,
    ]
      table [x=p, y=best_val_accuracy_median, y error=std_acc, col sep=comma]
      {data/mlp_all/ew_sparse_jac/all_layers.csv};

    \addplot[color=col_partial, mark=*, mark size=1.2pt]
      table [x=p, y=best_val_accuracy_median, col sep=comma]
      {data/mlp_all/partial_participation/all_layers.csv};
      \label{unif_plot:part_part}

    \addplot+[
      color=col_partial,
      draw=none,
      mark=none,
      forget plot,
      error bars/.cd, y dir=both, y explicit,
      error mark=none,
    ]
      table [x=p, y=best_val_accuracy_median, y error=std_acc, col sep=comma]
      {data/mlp_all/partial_participation/all_layers.csv};

    \addplot[color=col_uniform, mark=*, mark size=1.2pt]
      table [x=p, y=best_val_accuracy_median, col sep=comma]
      {data/mlp_all/uniform/all_layers.csv};
      \label{unif_plot:per_col}

    \addplot+[
      color=col_uniform,
      draw=none,
      mark=none,
      forget plot,
      error bars/.cd, y dir=both, y explicit,
      error mark=none,
    ]
      table [x=p, y=best_val_accuracy_median, y error=std_acc, col sep=comma]
      {data/mlp_all/uniform/all_layers.csv};

    \addplot[color=col_prop, mark=*, mark size=1.2pt]
      table [x=p, y=best_val_accuracy_median, col sep=comma]
      {data/mlp_all/corr_optimal_sec_mom/all_layers.csv};
      \label{unif_plot:opt}

    \addplot+[
      color=col_prop,
      draw=none,
      mark=none,
      forget plot,
      error bars/.cd, y dir=both, y explicit,
      error mark=none,
    ]
      table [x=p, y=best_val_accuracy_median, y error=std_acc, col sep=comma]
      {data/mlp_all/corr_optimal_sec_mom/all_layers.csv};

    \addplot[color=col_lemma, mark=*, mark size=1.2pt]
      table [x=p, y=best_val_accuracy_median, col sep=comma]
      {data/mlp_all/corr_diag_optimal_sec_mom/all_layers.csv};
      \label{unif_plot:diag}

    \addplot+[
      color=col_lemma,
      draw=none,
      mark=none,
      forget plot,
      error bars/.cd, y dir=both, y explicit,
      error mark=none,
    ]
      table [x=p, y=best_val_accuracy_median, y error=std_acc, col sep=comma]
      {data/mlp_all/corr_diag_optimal_sec_mom/all_layers.csv};

    \addplot[color=black, mark=star, only marks, mark size=4.5pt,
      mark options={line width=2pt}]
      coordinates {(1, 98.11)};
    
    \coordinate (legend) at (axis description cs:0.99,0.02);
    \end{axis}
    
\matrix [
          draw,
          rounded corners=1pt,
          fill=white,
          fill opacity=0.55,
          text opacity=1,
          matrix of nodes,
          anchor=south east,
          nodes={font=\tiny},
          column sep=-3pt,
          row sep=-3pt,
          cells ={anchor=west}
        ] at (legend) {
            Uniform Masks & [-4mm] & Data Dep. Sketches & [-4mm] \\
            Element & \ref{unif_plot:ew} & RCS (Prop. \ref{prop:optimal-column-sketch}) & \ref{unif_plot:diag} \\
            Sample & \ref{unif_plot:part_part}   & DS (Lem. \ref{cor:diag-mask}) & \ref{unif_plot:opt}\\
            Column & \ref{unif_plot:per_col} & & &\\
            Baseline & \tikz{\draw plot[only marks, mark=star, mark size=4.5pt, mark options={line width=1.5pt}] coordinates {(0,0)};} &  \\
        };

\end{tikzpicture}

%% file: figures_tex/scores_comparison.tex
\begin{tikzpicture}
  \begin{axis}[
    title={},
    title style={font=\small},
    width=0.9\columnwidth,
    height=0.52\columnwidth,
    xlabel={Sampling budget ($p$ or $r/n$)},
    ylabel={Best Accuracy (\%)},
    xlabel style={font=\footnotesize, inner sep=0pt},
    ylabel style={font=\footnotesize, inner sep=0pt},
    xlabel shift=4pt,
    ylabel shift=-4pt,
    grid=major,
    xmode=log,
    ymin=70, ymax=101,
    legend pos = south east,
    legend columns=1,
    legend style={font=\tiny},
    scale only axis,
    tick label style={font=\footnotesize},
    major tick length=2pt,
  ]
    \addplot[color=col_lone, mark=*, mark size=1.2pt, densely dashed]
      table [x=p, y=best_val_accuracy_median, col sep=comma]
      {data/mlp_all/corr_topr_l1/all_layers.csv};
      \label{score_plot:lone_sqrt}

    \addplot+[
      color=col_lone,
      draw=none,
      mark=none,
      forget plot,
      error bars/.cd, y dir=both, y explicit,
      error mark=none,
    ]
      table [x=p, y=best_val_accuracy_median, y error=std_acc, col sep=comma]
      {data/mlp_all/corr_topr_l1/all_layers.csv};

    \addplot[color=col_lone, mark=*, mark size=1.2pt]
      table [x=p, y=best_val_accuracy_median, col sep=comma]
      {data/mlp_all/corr_topr_l1_squared/all_layers.csv};
      \label{score_plot:lone}
      

    \addplot+[
      color=col_lone,
      draw=none,
      mark=none,
      forget plot,
      error bars/.cd, y dir=both, y explicit,
      error mark=none,
    ]
      table [x=p, y=best_val_accuracy_median, y error=std_acc, col sep=comma]
      {data/mlp_all/corr_topr_l1_squared/all_layers.csv};

    \addplot[color=col_ltwo, mark=*, mark size=1.2pt, densely dashed]
      table [x=p, y=best_val_accuracy_median, col sep=comma]
      {data/mlp_all/corr_topr_l2/all_layers.csv};
      \label{score_plot:ltwo_sqrt}

    \addplot+[
      color=col_ltwo,
      draw=none,
      mark=none,
      forget plot,
      error bars/.cd, y dir=both, y explicit,
      error mark=none,
    ]
      table [x=p, y=best_val_accuracy_median, y error=std_acc, col sep=comma]
      {data/mlp_all/corr_topr_l2/all_layers.csv};

    \addplot[color=col_ltwo, mark=*, mark size=1.2pt]
      table [x=p, y=best_val_accuracy_median, col sep=comma]
      {data/mlp_all/corr_topr_l2_squared/all_layers.csv};
      \label{score_plot:ltwo}

    \addplot+[
      color=col_ltwo,
      draw=none,
      mark=none,
      forget plot,
      error bars/.cd, y dir=both, y explicit,
      error mark=none,
    ]
      table [x=p, y=best_val_accuracy_median, y error=std_acc, col sep=comma]
      {data/mlp_all/corr_topr_l2_squared/all_layers.csv};

    \addplot[color=col_var, mark=*, mark size=1.2pt, densely dashed]
      table [x=p, y=best_val_accuracy_median, col sep=comma]
      {data/mlp_all/corr_topr_var/all_layers.csv};
      \label{score_plot:var_sqrt}

    \addplot+[
      color=col_var,
      draw=none,
      mark=none,
      forget plot,
      error bars/.cd, y dir=both, y explicit,
      error mark=none,
    ]
      table [x=p, y=best_val_accuracy_median, y error=std_acc, col sep=comma]
      {data/mlp_all/corr_topr_var/all_layers.csv};

    \addplot[color=col_var, mark=*, mark size=1.2pt,]
      table [x=p, y=best_val_accuracy_median, col sep=comma]
      {data/mlp_all/corr_topr_var_squared/all_layers.csv};
      \label{score_plot:var}

    \addplot+[
      color=col_var,
      draw=none,
      mark=none,
      forget plot,
      error bars/.cd, y dir=both, y explicit,
      error mark=none,
    ]
      table [x=p, y=best_val_accuracy_median, y error=std_acc, col sep=comma]
      {data/mlp_all/corr_topr_var_squared/all_layers.csv};

    \addplot[color=black, mark=star, only marks, mark size=4.5pt,
      mark options={line width=2pt}]
      coordinates {(1, 98.11)};
    
    \coordinate (legend) at (axis description cs:0.99,0.02);
\end{axis}

\matrix [
          draw,
          rounded corners=1pt,
          fill=white,
          fill opacity=0.55,
          text opacity=1,
          matrix of nodes,
          anchor=south east,
          nodes={font=\tiny},
          column sep=2pt,
          row sep=-4pt,
          cells ={anchor=west}
        ] at (legend) {
            Proxy                    & $\propto$ Proxy    & $\propto \sqrt{\text{Proxy}}$   \\
            $\ell_1$                 & \ref{score_plot:lone} & \ref{score_plot:lone_sqrt}\\
            $\ell_2$                 & \ref{score_plot:ltwo} & \ref{score_plot:ltwo_sqrt}\\
            $\mathrm{Var}$           & \ref{score_plot:var} & \ref{score_plot:var_sqrt}\\
            Baseline & \tikz{\draw plot[only marks, mark=star, mark size=4.5pt, mark options={line width=1.5pt}] coordinates {(0,0)};} &  \\
        };

\end{tikzpicture}

%% file: figures_tex/svd_vs_no_svd.tex
\begin{tikzpicture}
  \begin{axis}[
    title={},
    title style={font=\small},
    width=0.9\columnwidth,
    height=0.52\columnwidth,
    xlabel={Sampling budget ($p$ or $r/n$)},
    ylabel={Best Accuracy (\%)},
    xlabel style={font=\footnotesize, inner sep=0pt},
    ylabel style={font=\footnotesize, inner sep=0pt},
    xlabel shift=4pt,
    ylabel shift=-4pt,
    grid=major,
    xmode=log,
    ymin=77, ymax=101,
    legend pos = south east,
    legend columns=1,
    legend style={font=\tiny},
    scale only axis,
    tick label style={font=\footnotesize},
    major tick length=2pt,
  ]
    \addplot[color=col_gsv, mark=*, mark size=1.2pt, densely dashed]
      table [x=p, y=best_val_accuracy_median, col sep=comma]
      {data/mlp_all/corr_svd_topr/all_layers.csv};
      \label{svd_plot:gsv_sqrt}

    \addplot+[
      color=col_gsv,
      draw=none,
      mark=none,
      forget plot,
      error bars/.cd, y dir=both, y explicit,
      error mark=none,
    ]
      table [x=p, y=best_val_accuracy_median, y error=std_acc, col sep=comma]
      {data/mlp_all/corr_svd_topr/all_layers.csv};

    \addplot[color=col_prop, mark=*, mark size=1.2pt]
      table [x=p, y=best_val_accuracy_median, col sep=comma]
      {data/mlp_all/corr_optimal_sec_mom/all_layers.csv};
      \label{svd_plot:opt}

    \addplot+[
      color=col_prop,
      draw=none,
      mark=none,
      forget plot,
      error bars/.cd, y dir=both, y explicit,
      error mark=none,
    ]
      table [x=p, y=best_val_accuracy_median, y error=std_acc, col sep=comma]
      {data/mlp_all/corr_optimal_sec_mom/all_layers.csv};

    \addplot[color=col_lemma, mark=*, mark size=1.2pt]
      table [x=p, y=best_val_accuracy_median, col sep=comma]
      {data/mlp_all/corr_diag_optimal_sec_mom/all_layers.csv};
      \label{svd_plot:diag}

    \addplot+[
      color=col_lemma,
      draw=none,
      mark=none,
      forget plot,
      error bars/.cd, y dir=both, y explicit,
      error mark=none,
    ]
      table [x=p, y=best_val_accuracy_median, y error=std_acc, col sep=comma]
      {data/mlp_all/corr_diag_optimal_sec_mom/all_layers.csv};

    \addplot[color=col_lone, mark=*, mark size=1.2pt]
      table [x=p, y=best_val_accuracy_median, col sep=comma]
      {data/mlp_all/corr_topr_l1_squared/all_layers.csv};
      \label{svd_plot:l_one}

    \addplot+[
      color=col_lone,
      draw=none,
      mark=none,
      forget plot,
      error bars/.cd, y dir=both, y explicit,
      error mark=none,
    ]
      table [x=p, y=best_val_accuracy_median, y error=std_acc, col sep=comma]
      {data/mlp_all/corr_topr_l1_squared/all_layers.csv};

    \addplot[color=col_gsv, mark=*, mark size=1.2pt]
      table [x=p, y=best_val_accuracy_median, col sep=comma]
      {data/mlp_all/corr_svd_topr_squared/all_layers.csv};
      \label{svd_plot:gsv}

    \addplot+[
      color=col_gsv,
      draw=none,
      mark=none,
      forget plot,
      error bars/.cd, y dir=both, y explicit,
      error mark=none,
    ]
      table [x=p, y=best_val_accuracy_median, y error=std_acc, col sep=comma]
      {data/mlp_all/corr_svd_topr_squared/all_layers.csv};

    \addplot[color=black, mark=star, only marks, mark size=4.5pt,
      mark options={line width=2pt}]
      coordinates {(1, 98.11)};
    \coordinate (legend) at (axis description cs:0.99,0.02);
\end{axis}

\matrix [
          draw,
          rounded corners=1pt,
          fill=white,
          fill opacity=0.55,
          text opacity=1,
          matrix of nodes,
          anchor=south east,
          nodes={font=\tiny},
          column sep=-2pt,
          row sep=-3pt,
          cells ={anchor=west}
        ] at (legend) {
            Spectral Based & [-1.5mm] & Coordinate Based & [-4mm] \\
            $\sqrt{\text{G-SV}}$ & \ref{svd_plot:gsv_sqrt}   & DS (Lem. \ref{cor:diag-mask}) & \ref{svd_plot:diag} \\
            G-SV & \ref{svd_plot:gsv}   &  $\ell_1$ & \ref{svd_plot:l_one}\\
            RCS (Prop \ref{prop:optimal-column-sketch}) & \ref{svd_plot:opt} & & &\\
            Baseline & \tikz{\draw plot[only marks, mark=star, mark size=4.5pt, mark options={line width=1.5pt}] coordinates {(0,0)};} &  \\
        };

\end{tikzpicture}

%% file: figures_tex/all_methods_bagnet.tex
\begin{tikzpicture}
  \begin{axis}[
    title={},
    title style={font=\small},
    width=0.9\columnwidth,
    height=0.52\columnwidth,
    xlabel={Sampling budget ($p$ or $r/n$)},
    ylabel={Best Accuracy (\%)},
    xlabel style={font=\footnotesize, inner sep=0pt},
    ylabel style={font=\footnotesize, inner sep=0pt},
    xlabel shift=4pt,
    ylabel shift=-4pt,
    grid=major,
    xmode=log,
    ymin=12, ymax=93,
    legend pos = south east,
    legend columns=1,
    legend style={font=\tiny},
    scale only axis,
    tick label style={font=\footnotesize},
    major tick length=2pt,
  ]
   
    \addplot[color=col_lone, mark=*, mark size=1.2pt,]
      table [x=p, y=best_val_accuracy_median, col sep=comma]
      {data/bagnet/corr_topr_l1_squared/all_layers.csv};
      \label{bagnet_plot:l_one}

    \addplot+[
      color=col_lone,
      draw=none,
      mark=none,
      forget plot,
      error bars/.cd, y dir=both, y explicit,
    ]
      table [x=p, y=best_val_accuracy_median, y error=std_acc, col sep=comma]
      {data/bagnet/corr_topr_l1_squared/all_layers.csv};

    \addplot[color=col_lemma,mark=*, mark size=1.2pt,]
      table [x=p, y=best_val_accuracy_median, col sep=comma]
      {data/bagnet/corr_diag_optimal_sec_mom/all_layers.csv};
      \label{bagnet_plot:diag}

    \addplot+[
      color=col_lemma,
      draw=none,
      mark=none,
      forget plot,
      error bars/.cd, y dir=both, y explicit,
    ]
      table [x=p, y=best_val_accuracy_median, y error=std_acc, col sep=comma]
      {data/bagnet/corr_diag_optimal_sec_mom/all_layers.csv};

    \addplot[color=col_gsv,mark=*, mark size=1.2pt,]
      table [x=p, y=best_val_accuracy_median, col sep=comma]
      {data/bagnet/corr_svd_topr_squared/all_layers.csv};
      \label{bagnet_plot:gsv}

    \addplot+[
      color=col_gsv,
      draw=none,
      mark=none,
      forget plot,
      error bars/.cd, y dir=both, y explicit,
    ]
      table [x=p, y=best_val_accuracy_median, y error=std_acc, col sep=comma]
      {data/bagnet/corr_svd_topr_squared/all_layers.csv};

    \addplot[color=col_prop,mark=*, mark size=1.2pt,]
      table [x=p, y=best_val_accuracy_median, col sep=comma]
      {data/bagnet/corr_optimal_sec_mom/all_layers.csv};
      \label{bagnet_plot:opt}

    \addplot+[
      color=col_prop,
      draw=none,
      mark=none,
      forget plot,
      error bars/.cd, y dir=both, y explicit,
    ]
      table [x=p, y=best_val_accuracy_median, y error=std_acc, col sep=comma]
      {data/bagnet/corr_optimal_sec_mom/all_layers.csv};

    \addplot[color=col_ew_sparse, mark=*, mark size=1.2pt,]
      table [x=p, y=best_val_accuracy_median, col sep=comma]
      {data/bagnet/ew_sparse_jac/all_layers.csv};
      \label{bagnet_plot:ew}

    \addplot+[
      color=col_ew_sparse,
      draw=none,
      mark=none,
      forget plot,
      error bars/.cd, y dir=both, y explicit,
    ]
      table [x=p, y=best_val_accuracy_median, y error=std_acc, col sep=comma]
      {data/bagnet/ew_sparse_jac/all_layers.csv};

    \addplot[color=col_partial, mark=*, mark size=1.2pt,]
      table [x=p, y=best_val_accuracy_median, col sep=comma]
      {data/bagnet/partial_participation/all_layers.csv};
      \label{bagnet_plot:part_part}

    \addplot+[
      color=col_partial,
      draw=none,
      mark=none,
      forget plot,
      error bars/.cd, y dir=both, y explicit,
    ]
      table [x=p, y=best_val_accuracy_median, y error=std_acc, col sep=comma]
      {data/bagnet/partial_participation/all_layers.csv};

    \addplot[color=black, mark=star, only marks, mark size=4.5pt,
      mark options={line width=2pt}]
      coordinates {(1, 90.17)};
    \coordinate (legend) at (axis description cs:0.99,0.01);
    \end{axis}

\matrix [
          draw,
          rounded corners=1pt,
          fill=white,
          fill opacity=0.45,
          text opacity=1,
          matrix of nodes,
          anchor=south east,
          nodes={font=\tiny},
          column sep=-3pt,
          row sep=-3pt,
          cells ={anchor=west}
        ] at (legend) {
            Unif. Masks & Element & \ref{bagnet_plot:ew} \\
                        & Sample  & \ref{bagnet_plot:part_part} \\
            Coordinates & $\ell_1$    & \ref{bagnet_plot:l_one} \\
                        & DS (Lem. \ref{cor:diag-mask}) & \ref{bagnet_plot:diag} \\
            Spectral    & G-SV        & \ref{bagnet_plot:gsv} \\
                        & RCS (Prop. \ref{prop:optimal-column-sketch}) & \ref{bagnet_plot:opt} \\
            Baseline & \tikz{\draw plot[only marks, mark=star, mark size=4.5pt, mark options={line width=1.5pt}] coordinates {(0,0)};} &  \\
        };

\end{tikzpicture}

%% file: figures_tex/all_methods_vit.tex
\begin{tikzpicture}
  \begin{axis}[
    title={},
    title style={font=\small},
    width=0.9\columnwidth,
    height=0.52\columnwidth,
    xlabel={Sampling budget ($p$ or $r/n$)},
    ylabel={Best Accuracy (\%)},
    xlabel style={font=\footnotesize, inner sep=0pt},
    ylabel style={font=\footnotesize, inner sep=0pt},
    xlabel shift=4pt,
    ylabel shift=-4pt,
    grid=major,
    xmode=log,
    ymin=40, ymax=93,
    legend pos = south east,
    legend columns=1,
    legend style={font=\tiny},
    scale only axis,
    tick label style={font=\footnotesize},
    major tick length=2pt,
  ]
   
    \addplot[color=col_lone, mark=*, mark size=1.2pt,]
      table [x=p, y=best_val_accuracy_median, col sep=comma]
      {data/vit/corr_topr_l1_squared/all_layers.csv};
      \label{vit_plot:l_one}

    \addplot+[
      color=col_lone,
      draw=none,
      mark=none,
      forget plot,
      error bars/.cd, y dir=both, y explicit,
    ]
      table [x=p, y=best_val_accuracy_median, y error=std_acc, col sep=comma]
      {data/vit/corr_topr_l1_squared/all_layers.csv};

    \addplot[color=col_lemma,mark=*, mark size=1.2pt,]
      table [x=p, y=best_val_accuracy_median, col sep=comma]
      {data/vit/corr_diag_optimal_sec_mom/all_layers.csv};
      \label{vit_plot:diag}

    \addplot+[
      color=col_lemma,
      draw=none,
      mark=none,
      forget plot,
      error bars/.cd, y dir=both, y explicit,
    ]
      table [x=p, y=best_val_accuracy_median, y error=std_acc, col sep=comma]
      {data/vit/corr_diag_optimal_sec_mom/all_layers.csv};

    \addplot[color=col_gsv,mark=*, mark size=1.2pt,]
      table [x=p, y=best_val_accuracy_median, col sep=comma]
      {data/vit/corr_svd_topr_squared/all_layers.csv};
      \label{vit_plot:gsv}

    \addplot+[
      color=col_gsv,
      draw=none,
      mark=none,
      forget plot,
      error bars/.cd, y dir=both, y explicit,
    ]
      table [x=p, y=best_val_accuracy_median, y error=std_acc, col sep=comma]
      {data/vit/corr_svd_topr_squared/all_layers.csv};

    \addplot[color=col_prop,mark=*, mark size=1.2pt,]
      table [x=p, y=best_val_accuracy_median, col sep=comma]
      {data/vit/corr_optimal_sec_mom/all_layers.csv};
      \label{vit_plot:opt}

    \addplot+[
      color=col_prop,
      draw=none,
      mark=none,
      forget plot,
      error bars/.cd, y dir=both, y explicit,
    ]
      table [x=p, y=best_val_accuracy_median, y error=std_acc, col sep=comma]
      {data/vit/corr_optimal_sec_mom/all_layers.csv};

    \addplot[color=col_ew_sparse, mark=*, mark size=1.2pt,]
      table [x=p, y=best_val_accuracy_median, col sep=comma]
      {data/vit/ew_sparse_jac/all_layers.csv};
      \label{vit_plot:ew}

    \addplot+[
      color=col_ew_sparse,
      draw=none,
      mark=none,
      forget plot,
      error bars/.cd, y dir=both, y explicit,
    ]
      table [x=p, y=best_val_accuracy_median, y error=std_acc, col sep=comma]
      {data/vit/ew_sparse_jac/all_layers.csv};

    \addplot[color=col_partial, mark=*, mark size=1.2pt,]
      table [x=p, y=best_val_accuracy_median, col sep=comma]
      {data/vit/partial_participation/all_layers.csv};
      \label{vit_plot:part_part}

    \addplot+[
      color=col_partial,
      draw=none,
      mark=none,
      forget plot,
      error bars/.cd, y dir=both, y explicit,
    ]
      table [x=p, y=best_val_accuracy_median, y error=std_acc, col sep=comma]
      {data/vit/partial_participation/all_layers.csv};

    \addplot[color=black, mark=star, only marks, mark size=4.5pt,
      mark options={line width=2pt}]
      coordinates {(1, 90.17)};
    \coordinate (legend) at (axis description cs:0.99,0.01);
    \end{axis}

\matrix [
          draw,
          rounded corners=1pt,
          fill=white,
          fill opacity=0.45,
          text opacity=1,
          matrix of nodes,
          anchor=south east,
          nodes={font=\tiny},
          column sep=-3pt,
          row sep=-3pt,
          cells ={anchor=west}
        ] at (legend) {
            Unif. Masks & Element & \ref{bagnet_plot:ew} \\
                        & Sample  & \ref{bagnet_plot:part_part} \\
            Coordinates & $\ell_1$    & \ref{bagnet_plot:l_one} \\
                        & DS (Lem. \ref{cor:diag-mask}) & \ref{bagnet_plot:diag} \\
            Spectral    & G-SV        & \ref{bagnet_plot:gsv} \\
                        & RCS (Prop. \ref{prop:optimal-column-sketch}) & \ref{bagnet_plot:opt} \\
            Baseline & \tikz{\draw plot[only marks, mark=star, mark size=4.5pt, mark options={line width=1.5pt}] coordinates {(0,0)};} &  \\
        };

\end{tikzpicture}

%% file: figures_tex/variance_prop_masks.tex
\begin{tikzpicture}
  \begin{axis}[
    title={},
    title style={font=\small},
    width=0.9\columnwidth,
    height=0.62\columnwidth,
    xlabel={Sampling budget ($p$ or $r/n$)},
    ylabel={Best Accuracy (\%)},
    xlabel style={font=\footnotesize, inner sep=0pt},
    ylabel style={font=\footnotesize, inner sep=0pt},
    xlabel shift=4pt,
    ylabel shift=-4pt,
    grid=major,
    xmode=log,
    ymin=30, ymax=101,
    legend pos = south east,
    legend columns=1,
    legend style={font=\tiny},
    scale only axis,
    tick label style={font=\footnotesize},
    major tick length=2pt,
  ]
    \addplot[color=col_ew_sparse, mark=*]
      table [x=p, y=best_val_accuracy_median, col sep=comma]
      {data/mlp_all/ew_sparse_jac/all_layers.csv};
      \label{var_plot:ew_all}

    \addplot+[
      color=col_ew_sparse,
      draw=none,
      mark=none,
      forget plot,
      error bars/.cd, y dir=both, y explicit,
      error mark=none,
    ]
      table [x=p, y=best_val_accuracy_median, y error=std_acc, col sep=comma]
      {data/mlp_all/ew_sparse_jac/all_layers.csv};

    \addplot[color=col_ew_sparse, mark=*, densely dashed]
      table [x=p, y=best_val_accuracy_median, col sep=comma]
      {data/mlp_all/ew_sparse_jac/layer_4.csv};
      \label{var_plot:ew_4}

    \addplot+[
      color=col_ew_sparse,
      draw=none,
      mark=none,
      forget plot,
      error bars/.cd, y dir=both, y explicit,
      error mark=none,
    ]
      table [x=p, y=best_val_accuracy_median, y error=std_acc, col sep=comma]
      {data/mlp_all/ew_sparse_jac/layer_4.csv};

    \addplot[color=col_ew_sparse, mark=*, densely dotted]
      table [x=p, y=best_val_accuracy_median, col sep=comma]
      {data/mlp_all/ew_sparse_jac/layer_1.csv};
      \label{var_plot:ew_1}
      

    \addplot+[
      color=col_ew_sparse,
      draw=none,
      mark=none,
      forget plot,
      error bars/.cd, y dir=both, y explicit,
      error mark=none,
    ]
      table [x=p, y=best_val_accuracy_median, y error=std_acc, col sep=comma]
      {data/mlp_all/ew_sparse_jac/layer_1.csv};

    \addplot[color=col_partial, mark=*]
      table [x=p, y=best_val_accuracy_median, col sep=comma]
      {data/mlp_all/partial_participation/all_layers.csv};
      \label{var_plot:pp_all}
      

    \addplot+[
      color=col_partial,
      draw=none,
      mark=none,
      forget plot,
      error bars/.cd, y dir=both, y explicit,
      error mark=none,
    ]
      table [x=p, y=best_val_accuracy_median, y error=std_acc, col sep=comma]
      {data/mlp_all/partial_participation/all_layers.csv};

    \addplot[color=col_partial, mark=*, densely dashed]
      table [x=p, y=best_val_accuracy_median, col sep=comma]
      {data/mlp_all/partial_participation/layer_4.csv};
      \label{var_plot:pp_4}

    \addplot+[
      color=col_partial,
      draw=none,
      mark=none,
      forget plot,
      error bars/.cd, y dir=both, y explicit,
      error mark=none,
    ]
      table [x=p, y=best_val_accuracy_median, y error=std_acc, col sep=comma]
      {data/mlp_all/partial_participation/layer_4.csv};

    \addplot[color=col_partial, mark=*, densely dotted]
      table [x=p, y=best_val_accuracy_median, col sep=comma]
      {data/mlp_all/partial_participation/layer_1.csv};
      \label{var_plot:pp_1}

    \addplot+[
      color=col_partial,
      draw=none,
      mark=none,
      forget plot,
      error bars/.cd, y dir=both, y explicit,
      error mark=none,
    ]
      table [x=p, y=best_val_accuracy_median, y error=std_acc, col sep=comma]
      {data/mlp_all/partial_participation/layer_1.csv};

    \addplot[color=black, mark=star, only marks, mark size=6pt,
      mark options={line width=2pt}]
      coordinates {(1, 98.11)};
    
    \coordinate (legend) at (axis description cs:0.99,0.02);
    \end{axis}
    
\matrix [
          draw,
          rounded corners=1pt,
          fill=white,
          fill opacity=0.55,
          text opacity=1,
          matrix of nodes,
          anchor=south east,
          nodes={font=\tiny},
          column sep=-3pt,
          row sep=-3pt,
          cells ={anchor=west}
        ] at (legend) {
                   &         & VJP location &      \\
            Method & Layer 1 & Layer 4 & All-Layer \\
            Per-El. & \ref{var_plot:ew_1} & \ref{var_plot:ew_4} & \ref{var_plot:ew_all} \\
            Per-Sample. & \ref{var_plot:pp_1} & \ref{var_plot:pp_4} & \ref{var_plot:pp_all} \\
            Baseline & \tikz{\draw plot[only marks, mark=star, mark size=4.5pt, mark options={line width=1.5pt}] coordinates {(0,0)};} &  \\
        };

\end{tikzpicture}

%% file: figures_tex/variance_prop_coord.tex
\begin{tikzpicture}
  \begin{axis}[
    title={},
    title style={font=\small},
    width=0.9\columnwidth,
    height=0.62\columnwidth,
    xlabel={Sampling budget ($p$ or $r/n$)},
    ylabel={Best Accuracy (\%)},
    xlabel style={font=\footnotesize, inner sep=0pt},
    ylabel style={font=\footnotesize, inner sep=0pt},
    xlabel shift=4pt,
    ylabel shift=-4pt,
    grid=major,
    xmode=log,
    ymin=77, ymax=101,
    legend pos = south east,
    legend columns=1,
    legend style={font=\tiny},
    scale only axis,
    tick label style={font=\footnotesize},
    major tick length=2pt,
  ]
    \addplot[color=col_lemma, mark=*]
      table [x=p, y=best_val_accuracy_median, col sep=comma]
      {data/mlp_all/corr_diag_optimal_sec_mom/all_layers.csv};
      \label{var_plot:diag_all}

    \addplot+[
      color=col_lemma,
      draw=none,
      mark=none,
      forget plot,
      error bars/.cd, y dir=both, y explicit,
      error mark=none,
    ]
      table [x=p, y=best_val_accuracy_median, y error=std_acc, col sep=comma]
      {data/mlp_all/corr_diag_optimal_sec_mom/all_layers.csv};

    \addplot[color=col_lemma, mark=*, densely dashed]
      table [x=p, y=best_val_accuracy_median, col sep=comma]
      {data/mlp_all/corr_diag_optimal_sec_mom/layer_4.csv};
      \label{var_plot:diag_4}

    \addplot+[
      color=col_lemma,
      draw=none,
      mark=none,
      forget plot,
      error bars/.cd, y dir=both, y explicit,
      error mark=none,
    ]
      table [x=p, y=best_val_accuracy_median, y error=std_acc, col sep=comma]
      {data/mlp_all/corr_diag_optimal_sec_mom/layer_4.csv};

    \addplot[color=col_lemma, mark=*, densely dotted]
      table [x=p, y=best_val_accuracy_median, col sep=comma]
      {data/mlp_all/corr_diag_optimal_sec_mom/layer_1.csv};
      \label{var_plot:diag_1}

    \addplot+[
      color=col_lemma,
      draw=none,
      mark=none,
      forget plot,
      error bars/.cd, y dir=both, y explicit,
      error mark=none,
    ]
      table [x=p, y=best_val_accuracy_median, y error=std_acc, col sep=comma]
      {data/mlp_all/corr_diag_optimal_sec_mom/layer_1.csv};

    \addplot[color=col_lone, mark=*]
      table [x=p, y=best_val_accuracy_median, col sep=comma]
      {data/mlp_all/corr_topr_l1_squared/all_layers.csv};
      \label{var_plot:l_one_all}

    \addplot+[
      color=col_lone,
      draw=none,
      mark=none,
      forget plot,
      error bars/.cd, y dir=both, y explicit,
      error mark=none,
    ]
      table [x=p, y=best_val_accuracy_median, y error=std_acc, col sep=comma]
      {data/mlp_all/corr_topr_l1_squared/all_layers.csv};

    \addplot[color=col_lone, mark=*, densely dashed]
      table [x=p, y=best_val_accuracy_median, col sep=comma]
      {data/mlp_all/corr_topr_l1_squared/layer_4.csv};
      \label{var_plot:l_one_4}

    \addplot+[
      color=col_lone,
      draw=none,
      mark=none,
      forget plot,
      error bars/.cd, y dir=both, y explicit,
      error mark=none,
    ]
      table [x=p, y=best_val_accuracy_median, y error=std_acc, col sep=comma]
      {data/mlp_all/corr_topr_l1_squared/layer_4.csv};

    \addplot[color=col_lone, mark=*, densely dotted]
      table [x=p, y=best_val_accuracy_median, col sep=comma]
      {data/mlp_all/corr_topr_l1_squared/layer_1.csv};
      \label{var_plot:l_one_1}

    \addplot+[
      color=col_lone,
      draw=none,
      mark=none,
      forget plot,
      error bars/.cd, y dir=both, y explicit,
      error mark=none,
    ]
      table [x=p, y=best_val_accuracy_median, y error=std_acc, col sep=comma]
      {data/mlp_all/corr_topr_l1_squared/layer_1.csv};

    \addplot[color=black, mark=star, only marks, mark size=6pt,
      mark options={line width=2pt}]
      coordinates {(1, 98.11)};
    
    \coordinate (legend) at (axis description cs:0.99,0.02);
    \end{axis}
    
\matrix [
          draw,
          rounded corners=1pt,
          fill=white,
          fill opacity=0.55,
          text opacity=1,
          matrix of nodes,
          anchor=south east,
          nodes={font=\tiny},
          column sep=-3pt,
          row sep=-3pt,
          cells ={anchor=west}
        ] at (legend) {
                   &         & VJP location &      \\
            Method & Layer 1 & Layer 4 & All-Layer \\
            DS (Lem. \ref{cor:diag-mask}) & \ref{var_plot:diag_1} & \ref{var_plot:diag_4} & \ref{var_plot:diag_all} \\
            $\ell_1$ & \ref{var_plot:l_one_1} & \ref{var_plot:l_one_4} & \ref{var_plot:l_one_all} \\
            Baseline & \tikz{\draw plot[only marks, mark=star, mark size=4.5pt, mark options={line width=1.5pt}] coordinates {(0,0)};} &  \\
        };

\end{tikzpicture}

%% file: figures_tex/variance_prop_spectral.tex
\begin{tikzpicture}
  \begin{axis}[
    title={},
    title style={font=\small},
    width=0.9\columnwidth,
    height=0.62\columnwidth,
    xlabel={Sampling budget ($p$ or $r/n$)},
    ylabel={Best Accuracy (\%)},
    xlabel style={font=\footnotesize, inner sep=0pt},
    ylabel style={font=\footnotesize, inner sep=0pt},
    xlabel shift=4pt,
    ylabel shift=-4pt,
    grid=major,
    xmode=log,
    ymin=90, ymax=100,
    legend pos = south east,
    legend columns=1,
    legend style={font=\tiny},
    scale only axis,
    tick label style={font=\footnotesize},
    major tick length=2pt,
  ]
    \addplot[color=col_prop, mark=*]
      table [x=p, y=best_val_accuracy_median, col sep=comma]
      {data/mlp_all/corr_optimal_sec_mom/all_layers.csv};
      \label{var_plot:opti_all}

    \addplot+[
      color=col_prop,
      draw=none,
      mark=none,
      forget plot,
      error bars/.cd, y dir=both, y explicit,
      error mark=none,
    ]
      table [x=p, y=best_val_accuracy_median, y error=std_acc, col sep=comma]
      {data/mlp_all/corr_optimal_sec_mom/all_layers.csv};

    \addplot[color=col_prop, mark=*, densely dashed]
      table [x=p, y=best_val_accuracy_median, col sep=comma]
      {data/mlp_all/corr_optimal_sec_mom/layer_4.csv};
      \label{var_plot:opti_4}

    \addplot+[
      color=col_prop,
      draw=none,
      mark=none,
      forget plot,
      error bars/.cd, y dir=both, y explicit,
      error mark=none,
    ]
      table [x=p, y=best_val_accuracy_median, y error=std_acc, col sep=comma]
      {data/mlp_all/corr_optimal_sec_mom/layer_4.csv};

    \addplot[color=col_prop, mark=*, densely dotted]
      table [x=p, y=best_val_accuracy_median, col sep=comma]
      {data/mlp_all/corr_optimal_sec_mom/layer_1.csv};
      \label{var_plot:opti_1}
      

    \addplot+[
      color=col_prop,
      draw=none,
      mark=none,
      forget plot,
      error bars/.cd, y dir=both, y explicit,
      error mark=none,
    ]
      table [x=p, y=best_val_accuracy_median, y error=std_acc, col sep=comma]
      {data/mlp_all/corr_optimal_sec_mom/layer_1.csv};

    \addplot[color=col_gsv, mark=*]
      table [x=p, y=best_val_accuracy_median, col sep=comma]
      {data/mlp_all/corr_svd_topr_squared/all_layers.csv};
      \label{var_plot:gsv_all}
      

    \addplot+[
      color=col_gsv,
      draw=none,
      mark=none,
      forget plot,
      error bars/.cd, y dir=both, y explicit,
      error mark=none,
    ]
      table [x=p, y=best_val_accuracy_median, y error=std_acc, col sep=comma]
      {data/mlp_all/corr_svd_topr_squared/all_layers.csv};

    \addplot[color=col_gsv, mark=*, densely dashed]
      table [x=p, y=best_val_accuracy_median, col sep=comma]
      {data/mlp_all/corr_svd_topr_squared/layer_4.csv};
      \label{var_plot:gsv_4}

    \addplot+[
      color=col_gsv,
      draw=none,
      mark=none,
      forget plot,
      error bars/.cd, y dir=both, y explicit,
      error mark=none,
    ]
      table [x=p, y=best_val_accuracy_median, y error=std_acc, col sep=comma]
      {data/mlp_all/corr_svd_topr_squared/layer_4.csv};

    \addplot[color=col_gsv, mark=*, densely dotted]
      table [x=p, y=best_val_accuracy_median, col sep=comma]
      {data/mlp_all/corr_svd_topr_squared/layer_1.csv};
      \label{var_plot:gsv_1}

    \addplot+[
      color=col_gsv,
      draw=none,
      mark=none,
      forget plot,
      error bars/.cd, y dir=both, y explicit,
      error mark=none,
    ]
      table [x=p, y=best_val_accuracy_median, y error=std_acc, col sep=comma]
      {data/mlp_all/corr_svd_topr_squared/layer_1.csv};

    \addplot[color=black, mark=star, only marks, mark size=6pt,
      mark options={line width=2pt}]
      coordinates {(1, 98.11)};
    
    \coordinate (legend) at (axis description cs:0.99,0.02);
    \end{axis}
    
\matrix [
          draw,
          rounded corners=1pt,
          fill=white,
          fill opacity=0.55,
          text opacity=1,
          matrix of nodes,
          anchor=south east,
          nodes={font=\tiny},
          column sep=-3pt,
          row sep=-3pt,
          cells ={anchor=west}
        ] at (legend) {
                   &         & VJP location &      \\
            Method & Layer 1 & Layer 4 & All-Layer \\
            RCS (Prop. \ref{prop:optimal-column-sketch}) & \ref{var_plot:opti_1} & \ref{var_plot:opti_4} & \ref{var_plot:opti_all} \\
            G-SV & \ref{var_plot:gsv_1} & \ref{var_plot:gsv_4} & \ref{var_plot:gsv_all} \\
            Baseline & \tikz{\draw plot[only marks, mark=star, mark size=4.5pt, mark options={line width=1.5pt}] coordinates {(0,0)};} &  \\
        };

\end{tikzpicture}